\documentclass{article}

\usepackage[final]{nips_2016}

\usepackage{booktabs}
\usepackage{amsmath}
\usepackage{mathrsfs}
\usepackage{amsfonts}
\usepackage{amsthm}
\usepackage{booktabs}
\usepackage{stmaryrd}
\usepackage{algorithm}
\usepackage{url}
\usepackage{relsize}
\usepackage{color}

\usepackage{graphicx}

\newcommand{\hN}{{\hat N}}
\newcommand{\hn}{{\hat n}}
\newcommand{\dr}{{\dot r}}
\def \IG {\textrm{IG}}
\def \KL {\textrm{KL}}

\newcommand{\cX}{\mathcal{X}}
\newcommand{\cM}{\mathcal{M}}
\newcommand{\cA}{\mathcal{A}}

\newcommand{\bR}{\mathbb{R}}
\newcommand{\bN}{\mathbb{N}}

\newcommand{\indic}[1]{\mathbb{I} \left \{ #1 \right \} }

\DeclareMathOperator*{\expect}{{\huge \mathbb{E}}}
\DeclareMathOperator*{\smallexpect}{\mathbb{E}}
\newcommand{\expects}{\expect\nolimits}

\newtheorem{defn}{Definition}
\newtheorem{lem}{Lemma}

\newtheorem{assum}{Assumption}
\newtheorem{cor}{Corollary}
\newtheorem{thm}{Theorem}
\newtheorem{prop}{Proposition}

\newcommand{\xn}{x_{1:n}}

\newcommand{\cbar}{\, | \,}
\newcommand{\cdbar}{\, \| \,}
\newcommand{\csemi}{\, ; \,}

\def \gm {\textsc{dgm}}

\def \PG {\textrm{PG}}

\newcommand{\highlight}[1]{\noindent \textbf{#1}}

\newcommand{\eqnref}[1]{(\ref{eqn:#1})}

\newcommand{\gamename}[1]{\textsc{#1}}

\newcommand{\noappendix}[1]{#1}

\title{\smaller{Unifying Count-Based Exploration and Intrinsic Motivation}}

\author{
Marc G. Bellemare \\
bellemare@google.com \\
\And
Sriram Srinivasan \\
srsrinivasan@google.com \\
\And
Georg Ostrovski \\
ostrovski@google.com \\
\And
Tom Schaul \\
schaul@google.com \\
\And
David Saxton \\
saxton@google.com \\
\\
Google DeepMind \\
London, United Kingdom \\
\And
R\'emi Munos \\
munos@google.com \\
}

\begin{document}

\maketitle

\begin{abstract}
We consider an agent's uncertainty about its environment and the problem of generalizing this uncertainty across states. Specifically, we focus on the problem of exploration in non-tabular reinforcement learning. Drawing inspiration from the intrinsic motivation literature, we use density models to measure uncertainty, and propose a novel algorithm for deriving a pseudo-count from an arbitrary density model. This technique enables us to generalize count-based exploration algorithms to the non-tabular case. We apply our ideas to Atari 2600 games, providing sensible pseudo-counts from raw pixels. We transform these pseudo-counts into exploration bonuses and obtain significantly improved exploration in a number of hard games, including the infamously difficult \gamename{Montezuma's Revenge}.
\end{abstract}

\section{Introduction}

Exploration algorithms for Markov Decision Processes (MDPs) are typically concerned with reducing the agent's uncertainty over the environment's reward and transition functions. In a tabular setting, this uncertainty can be quantified using confidence intervals derived from Chernoff bounds, or inferred from a posterior over the environment parameters. In fact, both confidence intervals and posterior shrink as the inverse square root of the state-action visit count $N(x, a)$, making this quantity fundamental to most theoretical results on exploration.

Count-based exploration methods directly use visit counts to guide an agent's behaviour towards reducing uncertainty. For example, Model-based Interval Estimation with Exploration Bonuses \citep[MBIE-EB;][]{strehl08analysis} solves the augmented Bellman equation
\begin{equation*}
V(x) = \max_{a \in \cA} \left [ \hat R(x,a) + \gamma \smallexpect\nolimits_{\hat P} \left [ V(x') \right ] + \beta N(x,a)^{-1/2} \right ],
\end{equation*}
involving the empirical reward $\hat R$, the empirical transition function $\hat P$, and an exploration bonus proportional to $N(x,a)^{-1/2}$. This bonus accounts for uncertainties in both transition and reward functions and enables a finite-time bound on the agent's suboptimality.

In spite of their pleasant theoretical guarantees, count-based methods have not played a role in the contemporary successes of reinforcement learning \citep[e.g.][]{mnih15human}. Instead, most practical methods still rely on simple rules such as $\epsilon$-greedy. The issue is that visit counts are not directly useful in large domains, where states are rarely visited more than once.

Answering a different scientific question, intrinsic motivation aims to provide qualitative guidance for exploration \citep{schmidhuber91possibility,oudeyer07intrinsic,barto13intrinsic}. This guidance can be summarized as ``explore what surprises you''. A typical approach guides the agent based on change in prediction error, or \emph{learning progress}. If $e_n(A)$ is the error made by the agent at time $n$ over some event A, and $e_{n+1}(A)$ the same error after observing a new piece of information, then learning progress is
\begin{equation*}
e_n(A) - e_{n+1}(A) .
\end{equation*}
Intrinsic motivation methods are attractive as they remain applicable in the absence of the Markov property or the lack of a tabular representation, both of which are required by count-based algorithms. Yet the theoretical foundations of intrinsic motivation remain largely absent from the literature, which may explain its slow rate of adoption as a standard approach to exploration.

In this paper we provide formal evidence that intrinsic motivation and count-based exploration are but two sides of the same coin. Specifically, we consider a frequently used measure of learning progress, \emph{information gain} \citep{cover91elements}. Defined as the Kullback-Leibler divergence of a prior distribution from its posterior, information gain can be related to the confidence intervals used in count-based
exploration. Our contribution is to propose a new quantity, the \emph{pseudo-count}, which connects information-gain-as-learning-progress and count-based exploration.

We derive our pseudo-count from a density model over the state space. This is in departure from more traditional approaches to intrinsic motivation that consider learning progress with respect to a transition model.
We expose the relationship between pseudo-counts, a variant of \citeauthor{schmidhuber91possibility}'s compression progress we call \emph{prediction gain}, and information gain. Combined to \citeauthor{kolter09near}'s negative result on the frequentist suboptimality of Bayesian bonuses, our result highlights the theoretical advantages of pseudo-counts compared to many existing intrinsic motivation methods.

The pseudo-counts we introduce here are best thought of as ``function approximation for exploration''. We bring them to bear on Atari 2600 games from the Arcade Learning Environment \citep{bellemare13arcade}, focusing on games where myopic exploration fails. We extract our pseudo-counts from a simple density model and use them within a variant of MBIE-EB. 
We apply them to an experience replay setting and to an actor-critic setting, and find improved performance in both cases. Our approach produces dramatic progress on the reputedly most difficult Atari 2600 game, \gamename{Montezuma's Revenge}: within a fraction of the training time, our agent explores a significant portion of the first level and obtains significantly higher scores than previously published agents.

\section{Notation}

We consider a countable state space $\cX$.
We denote a sequence of length $n$ from $\cX$ by $\xn \in \cX^n$, the set of finite sequences from $\cX$ by $\cX^*$, write $\xn x$ to mean the concatenation of $\xn$ and a state $x \in \cX$, and denote the empty sequence by $\epsilon$.
A \emph{model} over $\cX$ is a mapping from $\cX^*$ to probability distributions over $\cX$. That is, for each $\xn \in \cX^n$ the model provides a probability distribution
\begin{equation*}
\rho_n(x) := \rho(x \csemi \xn) .
\end{equation*}
Note that we do not require $\rho_n(x)$ to be strictly positive for all $x$ and $\xn$. When it is, however, we may understand $\rho_n(x)$ to be the usual conditional probability of $X_{n+1} = x$ given $X_1 \dots X_n = \xn$.

We will take particular interest in the empirical distribution $\mu_n$ derived from the sequence $\xn$. If $N_n(x) := N(x,\xn)$ is the number of occurrences of a state $x$ in the sequence $\xn$, then
\begin{equation*}
\mu_n(x) := \mu(x \csemi \xn) := \frac{N_n(x)}{n} .
\end{equation*}
We call the $N_n$ the \emph{empirical count function}, or simply \emph{empirical count}.
The above notation extends to state-action spaces, and we write $N_n(x, a)$ to explicitly refer to the number of occurrences of a state-action pair when the argument requires it. 
When $\xn$ is generated by an ergodic Markov chain, for example if we follow a fixed policy in a finite-state MDP, then the limit point of $\mu_n$ is the chain's stationary distribution.

In our setting, a \emph{density model} is any model that assumes states are independently (but not necessarily identically) distributed; a density model is thus a particular kind of generative model. We emphasize that a density model differs from a forward model, which takes into account the temporal relationship between successive states. Note that $\mu_n$ is itself a density model.

\section{From Densities to Counts}\label{sec:prediction_to_counts}

In the introduction we argued that the visit count $N_n(x)$ (and consequently, $N_n(x,a)$) is not directly useful in practical settings, since states are rarely revisited. Specifically, $N_n(x)$ is almost always zero and cannot help answer the question ``How novel is this state?'' Nor is the problem solved by a Bayesian approach: even variable-alphabet models \citep[e.g.][]{hutter13sparse} must assign a small, diminishing probability to yet-unseen states.
To estimate the uncertainty of an agent's knowledge, we must instead look for a quantity which generalizes across states. Guided by ideas from the intrinsic motivation literature, we now derive such a quantity. We call it a \emph{pseudo-count} as it extends the familiar notion from Bayesian estimation.

\subsection{Pseudo-Counts and the Recoding Probability}

We are given a density model $\rho$ over $\cX$. This density model may be approximate, biased, or even inconsistent. We begin by introducing the \emph{recoding probability} of a state $x$: 
\begin{equation*}
\rho'_n(x) := \rho(x \csemi \xn x) .
\end{equation*}
This is the probability assigned to $x$ by our density model after observing a new 
occurrence of $x$. The term ``recoding'' is inspired from the statistical compression literature, where coding costs are inversely related to probabilities \citep{cover91elements}. When $\rho$ admits a conditional probability distribution,
\begin{equation*}
\rho'_n(x) = \Pr\nolimits_\rho(X_{n+2} = x \cbar X_1 \dots X_n = \xn, X_{n+1} = x) .
\end{equation*}
We now postulate two unknowns: a \emph{pseudo-count function} $\hN_n(x)$, and a \emph{pseudo-count total} $\hn$. We relate these two unknowns through two constraints: 
\begin{equation}\label{eqn:system_of_two_equations}
\rho_n(x) = \frac{\hN_n(x)}{\hn} \qquad \rho'_n(x) = \frac{\hN_n(x) + 1}{\hn + 1} .
\end{equation}
In words: we require that, after observing one instance of $x$, the density model's increase in prediction of that same $x$ should correspond to a unit increase in pseudo-count. The pseudo-count itself is derived from solving the linear system \eqnref{system_of_two_equations}: 
\begin{equation}\label{eqn:discrete_pseudo_count}
\hN_n(x) = \frac{\rho_n(x) (1 - \rho'_n(x))}{\rho'_n(x) - \rho_n(x)}  = \hn \rho_n(x) .
\end{equation}

Note that the equations \eqnref{system_of_two_equations} yield $\hN_n(x) = 0$ (with $\hn = \infty$) when $\rho_n(x) = \rho'_n(x) = 0$, and are inconsistent when $\rho_n(x) < \rho'_n(x) = 1$. 
These cases may arise from poorly behaved density models, but are easily accounted for. From here onwards we will assume a consistent system of equations.

\begin{defn}[Learning-positive density model]\label{defn:learning_positive}
A density model $\rho$ is \emph{learning-positive} if for all $\xn \in \cX^n$ and all $x \in \cX$, $\rho'_n(x) \ge \rho_n(x)$.
\end{defn}
By inspecting \eqnref{discrete_pseudo_count}, we see that
\begin{enumerate}
    \item{$\hN_n(x) \ge 0$ if and only if $\rho$ is learning-positive;}
    \item{$\hN_n(x) = 0$ if and only if $\rho_n(x) = 0$; and}
    \item{$\hN_n(x) = \infty$ if and only if $\rho_n(x) = \rho'_n(x)$.}
\end{enumerate}
In many cases of interest, the pseudo-count $\hN_n(x)$ matches our intuition. If $\rho_n = \mu_n$ then $\hN_n = N_n$. Similarly, if $\rho_n$ is a Dirichlet estimator then $\hN_n$ recovers the usual notion of pseudo-count. More importantly, if the  model generalizes across states then so do pseudo-counts.

\subsection{Estimating the Frequency of a Salient Event in \gamename{Freeway}}\label{sec:pseudo_counting_salient_events}

As an illustrative example, we employ our method to estimate the number of occurrences of an infrequent event in the Atari 2600 video game \gamename{Freeway} (Figure \ref{fig:pseudo_counts_atari}, screenshot). We use the Arcade Learning Environment \citep{bellemare13arcade}. We will demonstrate the following:
\begin{enumerate}
    \item{Pseudo-counts are roughly zero for novel events,}
    \item{they exhibit credible magnitudes,}
    \item{they respect the ordering of state frequency,}
    \item{they grow linearly (on average) with real counts,}
    \item{they are robust in the presence of nonstationary data.} 
\end{enumerate}
These properties suggest that pseudo-counts provide an appropriate generalized notion of visit counts in non-tabular settings.

In \gamename{Freeway}, the agent must navigate a chicken across a busy road. As our example, we consider estimating the number of times the chicken has reached the very top of the screen. As is the case for many Atari 2600 games, this naturally salient event is associated with an increase in score, which ALE translates into a positive reward. We may reasonably imagine that knowing how certain we are about this part of the environment is useful. After crossing, the chicken is teleported back to the bottom of the screen.

To highlight the robustness of our pseudo-count, we consider a nonstationary policy which waits for 250,000 frames, then applies the \textsc{up} action for 250,000 frames, then waits, then goes \textsc{up} again. The salient event only occurs during \textsc{up} periods. It also occurs with the cars in different positions, thus requiring generalization.
As a point of reference, we record the pseudo-counts for both the salient event and visits to the chicken's start position.

We use a simplified, pixel-level version of the CTS model for Atari 2600 frames proposed by \citet{bellemare14skip}, ignoring temporal dependencies. While the CTS model is rather impoverished in comparison to state-of-the-art density models for images \citep[e.g.][]{vandenoord16pixel},
its count-based nature results in extremely fast learning, making it an appealing candidate for exploration. Further details on the model may be found in the appendix.

\begin{figure*}[tb]
\center{
\includegraphics[width=4.0in]{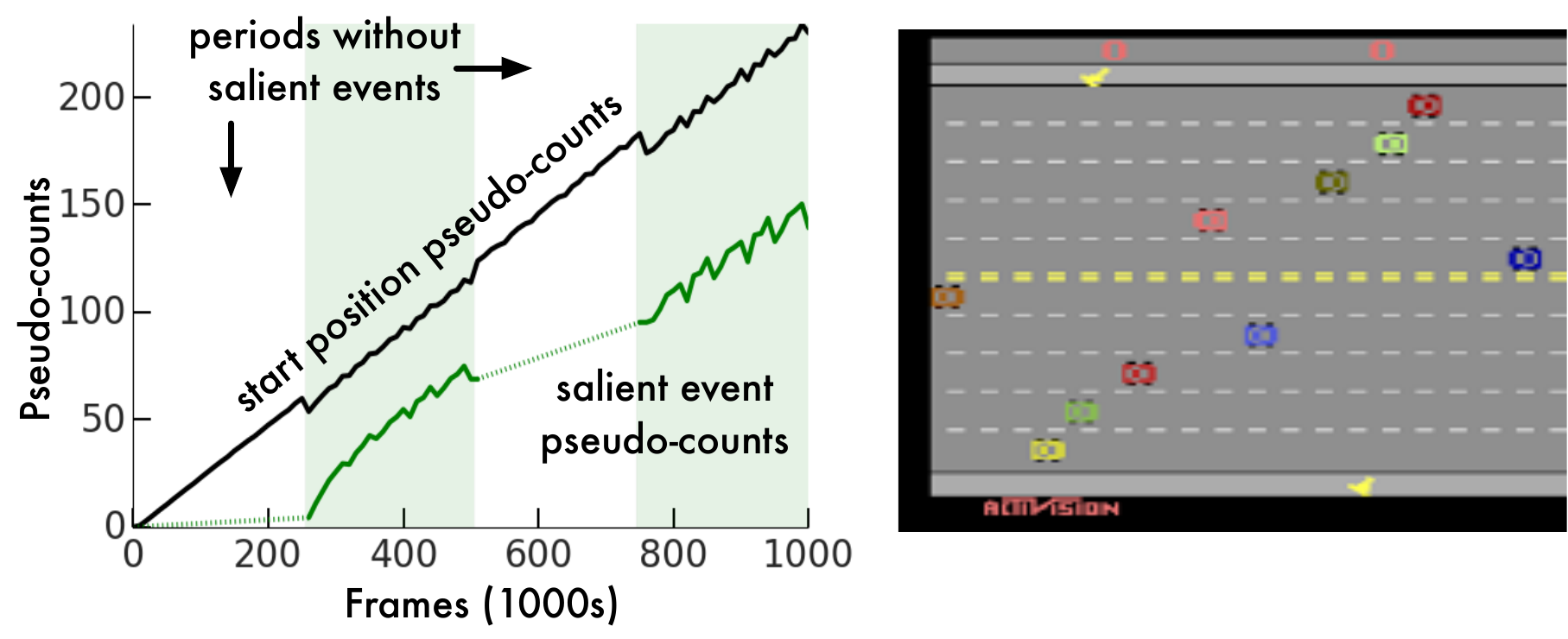}
}
\caption{
Pseudo-counts obtained from a CTS density model applied to \gamename{Freeway}, along with a frame representative of the salient event (crossing the road). Shaded areas depict periods during which the agent observes the salient event, dotted lines interpolate across periods during which the salient event is not observed. The reported values are 10,000-frame averages.\label{fig:pseudo_counts_atari}} 
\end{figure*}

Examining the pseudo-counts depicted in Figure \ref{fig:pseudo_counts_atari} confirms that they exhibit the desirable properties listed above. In particular, the pseudo-count is almost zero on the first occurrence of the salient event; it increases slightly during the 3rd period, since the salient and reference events share some common structure; throughout, it remains smaller than the reference pseudo-count. The linearity on average and robustness to nonstationarity are immediate from the graph. Note, however, that the pseudo-counts are a fraction of the real visit counts (inasmuch as we can define ``real''): by the end of the trial, the start position has been visited about 140,000 times, and the topmost part of the screen, 1285 times. Furthermore, the ratio of recorded pseudo-counts differs from the ratio of real counts. Both effects are quantifiable, as we shall show in Section \ref{sec:asymptotic_results}.

\section{The Connection to Intrinsic Motivation}\label{sec:connection_to_intrinsic_motivation}

Having argued that pseudo-counts appropriately generalize visit counts, we will now show that they are closely related to \emph{information gain}, which is commonly used to quantify novelty or curiosity and consequently as an intrinsic reward. 
Information gain is defined in relation to a \emph{mixture model} $\xi$ over a class of density models $\cM$. This model predicts according to a weighted combination from $\cM$:
\begin{equation*}
\xi_n(x) := \xi(x \csemi \xn) := \int_{\rho \in \cM} w_n(\rho) \rho(x \csemi \xn) \textrm{d}\rho ,
\end{equation*}
with $w_n(\rho)$ the posterior weight of $\rho$. This posterior is defined recursively, starting from a prior distribution $w_0$ over $\cM$:
\begin{equation} \label{eqn:posterior_update}
w_{n+1}(\rho) := w_{n}(\rho, x_{n+1}) \qquad w_n(\rho, x) := \frac{w_{n}(\rho) \rho(x \csemi \xn)}{\xi_n(x)} .
\end{equation}
Information gain is then the Kullback-Leibler divergence from prior to posterior that results from observing $x$:
\begin{equation*}
\IG_n(x) := \IG(x \csemi \xn) := \KL\big(w_n(\cdot, x) \cdbar w_n\big) .
\end{equation*}

Computing the information gain of a complex density model is often impractical, if not downright intractable. However, a quantity which we call the \emph{prediction gain} provides us with a good approximation of the information gain. We define the prediction gain of a density model $\rho$ (and in particular, $\xi$) as the difference between the recoding log-probability and log-probability of $x$:
\begin{equation*}
\PG_n(x) := \log \rho'_n(x) - \log \rho_n(x) .
\end{equation*}
Prediction gain is nonnegative if and only if $\rho$ is learning-positive. It is related to the pseudo-count:
\begin{equation*}
\hN_n(x) \approx \Big ( e^{\PG_n(x)} - 1 \Big )^{-1} ,
\end{equation*}
with equality when $\rho'_n(x) \to 0$. As the following theorem shows, prediction gain allows us to relate pseudo-count and information gain.
\begin{thm}\label{thm:main_result}
Consider a sequence $\xn \in \cX^n$. Let $\xi$ be a mixture model over a class of learning-positive models $\cM$. Let $\hN_n$ be the pseudo-count derived from $\xi$ (Equation \ref{eqn:discrete_pseudo_count}). For this model,
\begin{equation*}
\IG_n(x) \le \PG_n(x) \le \hN_n(x)^{-1} \qquad \text{ and } \qquad \; \PG_n(x) \le \hN_n(x)^{-1/2} .
\end{equation*}
\end{thm}
Theorem 1 suggests that using an exploration bonus proportional to $\hN_n(x)^{-1/2}$, similar to the MBIE-EB bonus, leads to a behaviour at least as exploratory as one derived from an information gain bonus. Since pseudo-counts correspond to empirical counts in the tabular setting, this approach also preserves known theoretical guarantees. In fact, we are confident pseudo-counts may be used to prove similar results in non-tabular settings.

On the other hand, it may be difficult to provide theoretical guarantees about existing bonus-based intrinsic motivation approaches.
\citet{kolter09near} showed that no algorithm based on a bonus upper bounded by $\beta N_n(x)^{-1}$ for any $\beta > 0$ can guarantee PAC-MDP optimality. Again considering the tabular setting and combining their result to Theorem \ref{thm:main_result}, we conclude that bonuses proportional to immediate information (or prediction) gain are insufficient for theoretically near-optimal exploration: to paraphrase \citeauthor{kolter09near}, these methods produce explore too little in comparison to pseudo-count bonuses. By inspecting (\ref{eqn:discrete_pseudo_count}) we come to a similar negative conclusion for bonuses proportional to the L1 or L2 distance between $\xi'_n$ and $\xi_n$.

Unlike many intrinsic motivation algorithms, pseudo-counts also do not rely on learning a forward (transition and/or reward) model. This point is especially important because a number of powerful density models for images exist \citep{vandenoord16pixel}, and because optimality guarantees cannot in general exist for intrinsic motivation algorithms based on forward models.

\section{Asymptotic Analysis}\label{sec:asymptotic_results}

In this section we analyze the limiting behaviour of the ratio $\hN_n / N_n$. We use this analysis to assert the consistency of pseudo-counts derived from tabular density models, i.e. models which maintain per-state visit counts. In the appendix we use the same result to bound the approximation error of pseudo-counts derived from directed graphical models, of which our CTS model is a special case.

Consider a fixed, infinite sequence $x_1, x_2, \dots$ from $\cX$. 
We define the limit of a sequence of functions $\big ( f(x \csemi \xn) \,: \, n \in \bN \big)$ with respect to the length $n$ of the subsequence $\xn$.
We additionally assume that the empirical distribution $\mu_n$ converges pointwise to a distribution $\mu$,
and write $\mu'_n(x)$ for the recoding probability of $x$ under $\mu_n$. We begin with two assumptions on our density model.
\begin{assum}\label{assum:limit_assumptions}
The limits 
\begin{equation*}
\textnormal{(a)  } r(x) := \lim_{n \to \infty} \frac{\rho_n(x)}{\mu_n(x)} \qquad \textnormal{(b)  } \dr(x) := \lim_{n \to \infty} \frac{\rho'_n(x) - \rho_n(x)}{\mu'_n(x) - \mu_n(x)}
\end{equation*}
exist for all $x$; furthermore, $\dr(x) > 0$. 
\end{assum}
Assumption (a) states that $\rho$ should eventually assign a probability to $x$ proportional to the limiting empirical distribution $\mu(x)$.
In particular there must be a state $x$ for which $r(x) < 1$, unless $\rho_n \to \mu$.
Assumption (b), on the other hand, imposes a restriction
on the learning rate of $\rho$ relative to $\mu$'s.
As both $r(x)$ and $\mu(x)$ exist, Assumption \ref{assum:limit_assumptions} also implies
that $\rho_n(x)$ and $\rho'_n(x)$ have a common limit.

\begin{thm}\label{thm:ratio_of_counts}
Under Assumption \ref{assum:limit_assumptions}, the limit of the ratio of pseudo-counts $\hN_n(x)$ to empirical counts $N_n(x)$ exists for all $x$. This limit is 
\begin{equation*}
\lim_{n \to \infty} \frac{\hN_n(x)}{N_n(x)} = \frac{r(x)}{\dr(x)} \left (\frac{1 - \mu(x) r(x)}{1 - \mu(x)} \right ) . 
\end{equation*}
\end{thm}

The model's relative rate of change, whose convergence to $\dr(x)$ we require, plays an essential role in
the ratio of pseudo- to empirical counts. To see this,
consider a sequence $(x_n : n \in \bN)$ generated i.i.d. from a distribution $\mu$ over a finite state space, and a density model defined from a sequence of nonincreasing step-sizes $(\alpha_n : n \in \bN)$:
\begin{equation*}
\rho_n(x) = (1 - \alpha_n) \rho_{n-1}(x) + \alpha_n \indic{x_n = x},
\end{equation*}
with initial condition $\rho_0(x) = |\cX|^{-1}$. For $\alpha_n = n^{-1}$, this density
model is the empirical distribution. For $\alpha_n = n^{-2/3}$, we may appeal to well-known results from stochastic approximation \citep[e.g.][]{bertsekas96neurodynamic} and find that almost surely 
\begin{equation*}
\lim_{n \to \infty} \rho_n(x) = \mu(x) \qquad \text{but} \qquad \lim_{n \to \infty} \frac{\rho'_n(x) - \rho_n(x)}{\mu'_n(x) - \mu_n(x)} = \infty .
\end{equation*}
Since $\mu'_n(x) - \mu_n(x) = n^{-1} (1 - \mu'_n(x))$, we may think of Assumption 1(b) as also requiring $\rho$ to converge at a rate of $\Theta(1/n)$ for a comparison with the empirical count $N_n$ to be meaningful. Note, however, that a density model that does not satisfy Assumption 1(b) may still yield useful (but incommensurable) pseudo-counts.

\begin{cor}\label{cor:atomic_models}
Let $\phi(x) > 0$ with $\sum_{x \in \cX} \phi(x) < \infty$ and consider the count-based estimator
\begin{equation*}
\rho_n(x) = \frac{N_n(x) + \phi(x)}{n + \sum_{x' \in \cX} \phi(x')} .
\end{equation*}
If $\hN_n$ is the pseudo-count corresponding to $\rho_n$ then $\hN_n(x) / N_n(x) \to 1$ for all $x$
with $\mu(x) > 0$.
\end{cor}

\section{Empirical Evaluation}

In this section we demonstrate the use of pseudo-counts to guide exploration. We return to the
Arcade Learning Environment, now using the CTS model to generate an exploration bonus.

\subsection{Exploration in Hard Atari 2600 Games}

From 60 games available through the Arcade Learning Environment we selected five ``hard'' games,
in the sense that an $\epsilon$-greedy policy is inefficient at exploring them. We used a bonus
of the form
\begin{equation}\label{eqn:square_root_bonus}
R^+_n(x,a) := \beta (\hN_n(x) + 0.01)^{-1/2},  
\end{equation}
where $\beta = 0.05$ was selected from a coarse parameter sweep. 
We also compared our method to the optimistic initialization trick proposed by
\citet{machado14domainindependent}.
We trained our agents' Q-functions
with Double DQN \citep{vanhasselt16deep}, with one important modification: we mixed the 
Double Q-Learning target
with the Monte Carlo return. This modification led to improved results both with
and without exploration bonuses (details in the appendix).

\begin{figure*}
\center{
\includegraphics[width=\textwidth,trim={0.25cm 0.25cm 0.25cm 0.25cm},clip]{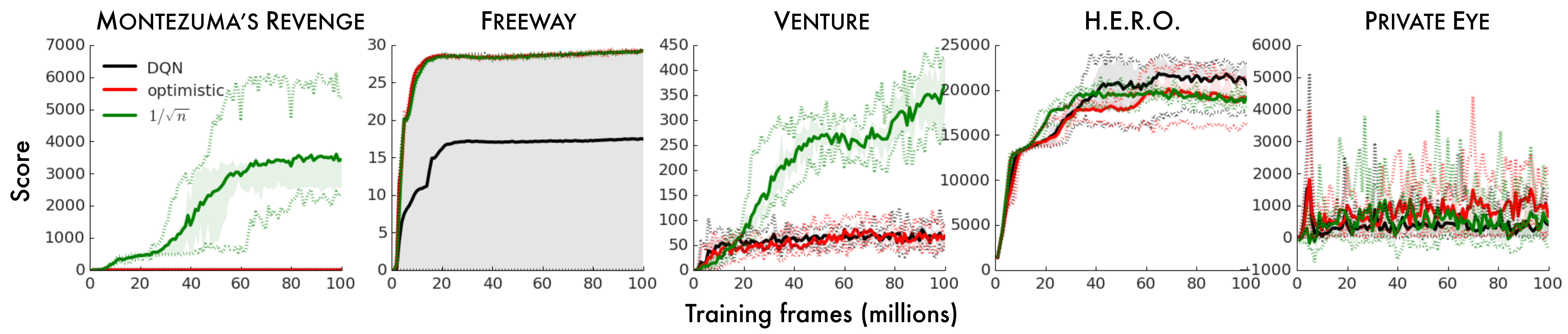}
\vspace{-1em}
}
\caption{Average training score with and without exploration bonus or optimistic initialization in 5 Atari 2600 games. Shaded areas denote inter-quartile range, dotted lines show min/max scores. 
\label{fig:5_hard}} 
\end{figure*}

Figure \ref{fig:5_hard} depicts the result of our experiment, averaged across 5 trials.
Although optimistic initialization helps in \gamename{Freeway}, it otherwise yields performance
similar to DQN. By contrast, the count-based exploration bonus enables us to make quick progress
on a number of games, most dramatically in \gamename{Montezuma's Revenge} and \gamename{Venture}. 

\gamename{Montezuma's Revenge} is perhaps the hardest Atari 2600 game available
through the ALE. The game is infamous for its hostile, unforgiving
environment: the agent must navigate a number of different rooms, each filled with traps.
Due to its sparse reward function, most published agents achieve an average score close to zero and completely fail to explore most of the 24 rooms that constitute the first level (Figure \ref{fig:montezuma_revenge_map_at_50_main}, top).
By contrast, within 50 million frames our agent learns a policy which consistently navigates
through 15 rooms (Figure \ref{fig:montezuma_revenge_map_at_50_main}, bottom). Our agent also achieves a score higher than anything previously reported, with one run consistently achieving 6600 points by 100 million frames (half the training samples used by \citet{mnih15human}). We believe the success of our method in this game is a strong indicator of the 
usefulness of pseudo-counts for exploration.\footnote{A video of our agent playing is available at \url{https://youtu.be/0yI2wJ6F8r0}.} 

\begin{figure*}
\center{
\includegraphics[width=5.5in]{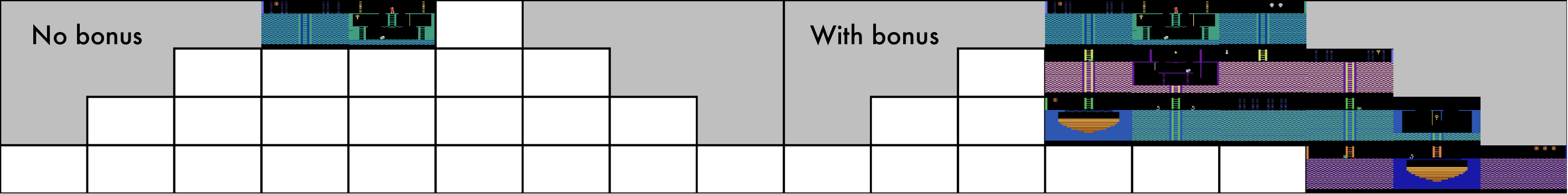}
\vspace{-1em}
}
\caption{``Known world'' of a DQN agent trained for 50 million frames with (\textbf{right}) and without (\textbf{left}) count-based exploration bonuses, in \gamename{Montezuma's Revenge}.\label{fig:montezuma_revenge_map_at_50_main}}
\end{figure*}

\subsection{Exploration for Actor-Critic Methods}

We next used our exploration bonuses in conjunction with the A3C 
(Asynchronous Advantage Actor-Critic) algorithm of \citet{mnih16asynchronous}. One appeal of
actor-critic methods is their explicit separation of policy and Q-function parameters, which leads to a richer behaviour space. This very separation, however, often leads to deficient exploration: to produce
any sensible results, the A3C policy must be regularized with an entropy cost.
We trained A3C on 60 Atari 2600 games, with and without the exploration bonus \eqnref{square_root_bonus}. We refer to our augmented algorithm as A3C+. Full details and additional results may be found in the appendix. 

We found that A3C fails to learn in \textbf{15} games, in the sense
that the agent does not achieve a score 50\% better than random. In comparison, there
are only \textbf{10} games for which A3C+ fails to improve on the random agent; of these, 
\textbf{8} are games where DQN fails in the same sense. We normalized the two algorithms' 
scores so that 0 and 1 are respectively the minimum and maximum of the random agent's and A3C's
end-of-training score on a particular game. Figure \ref{fig:a3c_agg} depicts the in-training median score for A3C and A3C+, along with 1st and 3rd quartile intervals. Not only does A3C+ achieve
slightly superior median performance, but it also significantly outperforms A3C on at least a
quarter of the games.  This is particularly important given the large proportion of Atari 2600
games for which an $\epsilon$-greedy policy is sufficient for exploration. 

\begin{figure*}
\center{
\includegraphics[height=1.5in]{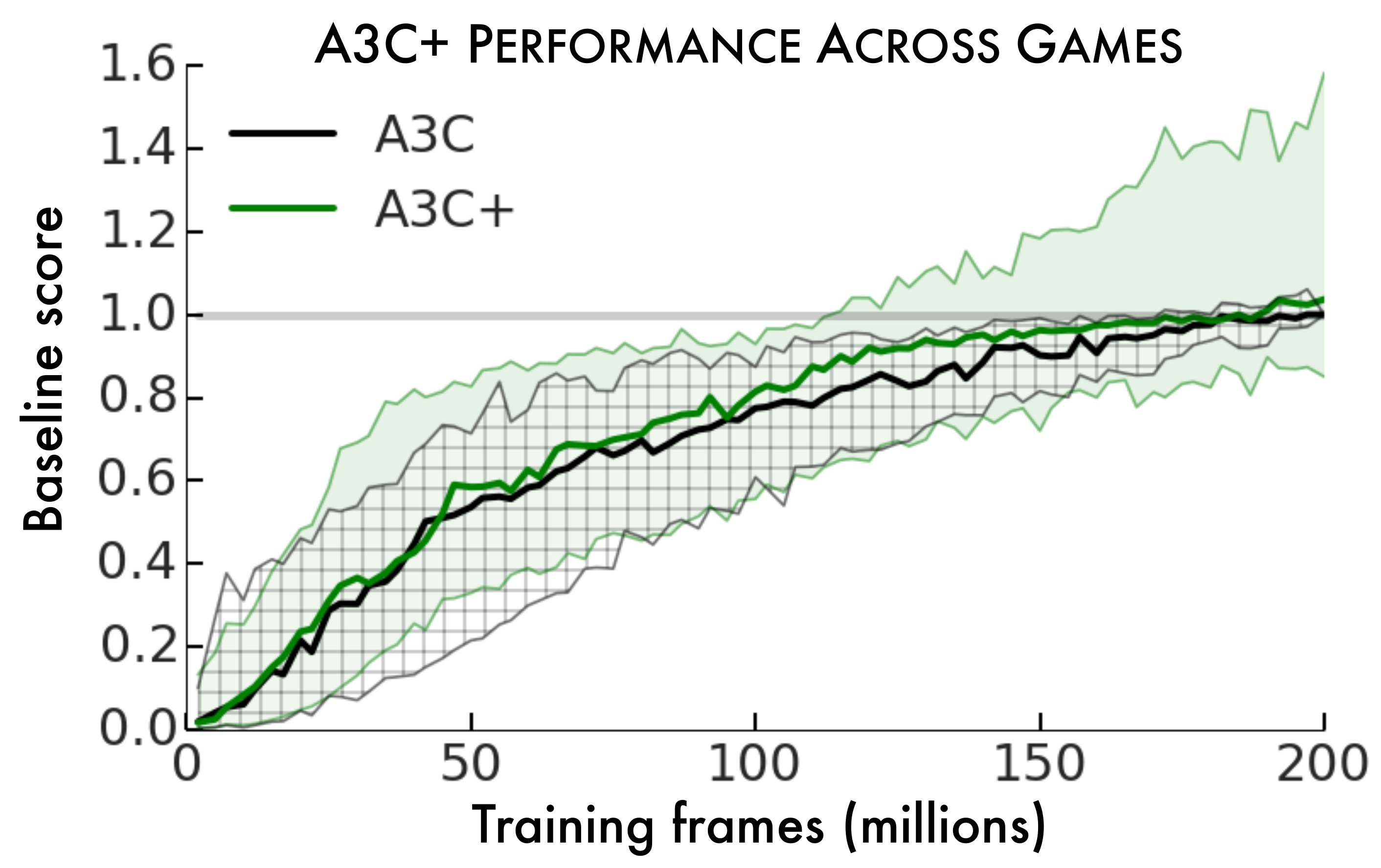}
}
\caption{Median and interquartile performance across 60 Atari 2600 games for A3C and A3C+.
\label{fig:a3c_agg}} 
\end{figure*}

\section{Related Work}

Information-theoretic quantities have been repeatedly used to describe intrinsically motivated behaviour. Closely related to prediction gain is \citet{schmidhuber91possibility}'s notion of compression progress, which equates novelty with an agent's improvement in its ability to compress its past.
More recently, \citet{lopes12exploration} showed the relationship between time-averaged prediction gain and visit counts in a tabular setting; their result is a special case of Theorem \ref{thm:ratio_of_counts}. \citet{orseau13universal} demonstrated that maximizing the sum of future information gains does lead to optimal behaviour, even though maximizing immediate information gain does not (Section \ref{sec:connection_to_intrinsic_motivation}). Finally, there may be a connection between sequential normalized maximum likelihood estimators and our pseudo-count derivation \citep[see e.g.][]{ollivier15laplace}.

Intrinsic motivation has also been studied in reinforcement learning proper, in particular in the context of discovering skills \citep{singh04intrinsically,barto13intrinsic}. Recently, \citet{stadie15incentivizing} used a squared prediction error bonus for exploring in Atari 2600 games. Closest to our work is \citet{houthooft16curiosity}'s variational approach to intrinsic motivation, which is equivalent to a second order Taylor approximation to prediction gain.
\citet{mohamed15variational} also considered a variational approach to the different problem of maximizing an agent's ability to influence its environment.

Aside for \citeauthor{orseau13universal}'s above-cited work, it is only recently that theoretical guarantees for exploration have emerged for non-tabular, stateful settings. We note \citet{pazis16efficient}'s PAC-MDP result for metric spaces and \citet{leike16thompson}'s asymptotic analysis of Thompson sampling in general environments.

\section{Future Directions}\label{sec:conclusion}

The last few years have seen tremendous advances in learning representations for 
reinforcement learning. Surprisingly, these advances have yet to carry over to 
the problem of exploration. In this paper, we reconciled counts, the fundamental unit of
uncertainty, with prediction-based heuristics and intrinsic motivation. 
Combining our work with more ideas from deep learning and better density models seems a plausible avenue for quick progress in practical, efficient exploration.
We now conclude by outlining a few research directions we believe are promising.

\highlight{Induced metric.} We did not address the question of \emph{where} the generalization comes from. Clearly, the choice of density model induces a particular metric over the state space. A better understanding of this metric should allow us to tailor the density model to the problem of exploration. 

\highlight{Compatible value function.} There may be a mismatch in the learning rates of the density model and the value function: DQN learns much more slowly than our
CTS model. As such, it should be beneficial to design value functions
compatible with density models (or vice-versa). 

\highlight{The continuous case.} Although we focused here on countable state spaces, we can as easily define a pseudo-count in terms of probability density functions. 
At present it is unclear whether this provides us with the right notion of counts for continuous spaces. 

\subsubsection*{Acknowledgments}

The authors would like to thank Laurent Orseau, Alex Graves, Joel Veness, Charles Blundell, Shakir Mohamed, Ivo Danihelka, Ian Osband, Matt Hoffman, Greg Wayne, Will Dabney, and A\"aron van den Oord for their excellent feedback early and late in the writing, and Pierre-Yves Oudeyer and Yann Ollivier for pointing out additional connections to the literature. 

\bibliographystyle{apalike}
\begin{small}
\bibliography{dbim}
\end{small}

\noappendix{
\appendix

\section{The Connection to Intrinsic Motivation}

The following provides an identity connecting information gain and prediction gain.
\begin{lem}\label{lem:identity_ipd_ig}
Consider a mixture model $\xi$ over $\cM$ with prediction gain $\PG_n$ and information gain $\IG_n$, a fixed $x \in \cX$, and let $w'_n(x) := w_n(\rho, x)$ be the posterior of $\xi$ over $\cM$ after observing $x$. Let $w''_n(x) := w'_n(\rho, x)$ be the same posterior after observing $x$ a second time, and let $\PG^\rho_n(x)$ denote the prediction gain of $\rho \in \cM$. Then
\begin{equation*}
\PG_n(x) = \KL(w'_n \cdbar w_n) + \KL(w'_n \cdbar w''_n) = \IG_n(x) + \KL(w'_n \cdbar w''_n) + \expects_{w'_n} \left [ \PG^{\rho}_{n}(x) \right ].
\end{equation*}
In particular, if $\cM$ is a class of non-adaptive models in the sense that $\rho_n(x) = \rho(x)$ for all $\xn$, then
\begin{equation*}
\PG_n(x) = \KL(w'_n \cdbar w_n) + \KL(w'_n \cdbar w''_n) = \IG_n(x) + \KL(w'_n \cdbar w''_n) . 
\end{equation*}
\end{lem}

A model which is non-adaptive is also learning-positive in the sense of Definition \ref{defn:learning_positive}. Many common mixture models, for example Dirichlet-multinomial estimators, are mixtures over non-adaptive models.

\begin{proof}
We rewrite the posterior update rule \eqnref{posterior_update} to show that for any $\rho \in \cM$ and any $x \in \cX$,
\begin{equation*}
\xi_n(x) = \frac{\rho_n(x) w_n(\rho)}{w_n(\rho, x)} .
\end{equation*}
Write $\expects_{w'_n} := \expect_{\rho \sim w'_n(\cdot)}$. Now
\begin{align*}
PG_n(x) = \log \frac{\xi'_n(x)}{\xi_n(x)} &= \expects_{w'_n} \left [ \log \frac{\xi'_n(x)}{\xi_n(x)} \right ] \\
&= \expects_{w'_n} \left [ \log \frac{w'_n(\rho)}{w''_n(\rho)} \frac{w'_n(\rho)}{w_n(\rho)} \frac{\rho'_n(x)}{\rho_n(x)} \right ] \\
&= \expects_{w'_n} \left [ \log \frac{w'_n(\rho)}{w_n(\rho)} \right ] + \expects_{w'_n} \left [ \log \frac{w'_n(\rho)}{w''_n(\rho)} \right ] + \expects_{w'_n} \left [ \log \frac{\rho'_n(x)}{\rho_n(x)} \right ] \\
& = \IG_n(x) + \KL(w'_n \cdbar w''_n) + \expects_{w'_n} \left [ \PG^{\rho}_{n}(x) \right ] . \qedhere
\end{align*}
The second statement follows immediately.
\end{proof}

\begin{lem}\label{lem:logarithmic_inequality}
The functions $f(x) := e^x - 1 - x$ and $g(x) := e^x - 1 - x^2$ are nonnegative on $x \in [0, \infty)$.
\end{lem}
\begin{proof}
The statement regarding $f(x)$ follows directly from the Taylor expansion for $e^x$. Now, the first derivative of $g(x)$ is $e^x - 2x$. It is clearly positive for $x \ge 1$. For $x \in [0, 1]$,
\begin{equation*}
e^x - 2x = \sum_{i=0}^\infty \frac{x^i}{i!} - 2x \ge 1 - x \ge 0 . \\
\end{equation*}
Since $g(0) = 0$, the second result follows.
\end{proof}
\begin{proof}[Proof (Theorem \ref{thm:main_result})]
The inequality $\IG_n(x) \le \PG_n(x)$ follows directly from Lemma \ref{lem:identity_ipd_ig}, 
the nonnegativity of the Kullback-Leibler divergence, and the fact that all models in $\cM$ are learning-positive. For the inequality $\PG_n(x) \le \hN_n(x)^{-1}$,
we write 
\begin{align*}
\hN_n(x)^{-1} &= (1 - \xi'_n(x))^{-1} \frac{\xi'_n(x) - \xi_n(x)}{\xi_n(x)} \\
&= (1 - \xi'_n(x))^{-1} \left ( \frac{\xi'_n(x)}{\xi_n(x)} - 1 \right ) \\
&\overset{(a)}{=} (1 - \xi'_n(x))^{-1} \left ( e^{\PG_n(x)} - 1 \right ) \\
&\overset{(b)}{\ge} e^{\PG_n(x)} - 1 \\
&\overset{(c)}{\ge} \PG_n(x),
\end{align*}
where (a) follows by definition of prediction gain, (b) from $\xi'_n(x) \in [0, 1)$, and (c) from Lemma \ref{lem:logarithmic_inequality}. Using the second part of Lemma \ref{lem:logarithmic_inequality} in (c) yields the inequality $\hN_n(x)^{-1/2} \ge \PG_n(x)$.
\end{proof}

\section{Asymptotic Analysis}

We begin with a simple lemma which will prove useful throughout.
\begin{lem}\label{lem:empirical_rate_of_change}
The rate of change of the empirical distribution, $\mu'_n(x) - \mu_n(x)$, is such that
\begin{equation*}
n \big ( \mu'_n(x) - \mu_n(x) \big ) = 1 - \mu'_n(x) . 
\end{equation*}
\end{lem}
\begin{proof}
We expand the definition of $\mu_n$ and $\mu'_n$:
\begin{align*}
n \big (\mu'_n(x) - \mu_n(x) \big ) &= n \left [ \frac{N_n(x) + 1}{n + 1} - \frac{N_n(x)}{n} \right ] \\
&= \left [ \frac{n}{n+1} \big ( N_n(x) + 1 \big ) - N_n(x) \right ] \\
&= \left [ 1 - \frac{N_n(x) + 1}{n + 1} \right ] \\
&= 1 - \mu'_n(x) . 
\end{align*}
\end{proof}

Using this lemma, we derive an asymptotic relationship between $N_n$ and $\hN_n$. 

\begin{proof}[Proof (Theorem \ref{thm:ratio_of_counts})]
We expand the definition of $\hN_n(x)$ and $N_n(x)$: 
\begin{align*}
\frac{\hN_n(x)}{N_n(x)} &= \frac{\rho_n(x) (1 - \rho'_n(x))}{N_n(x) (\rho'_n(x) - \rho_n(x))} \\
&= \frac{\rho_n(x) (1 - \rho'_n(x))}{n \mu_n(x) (\rho'_n(x) - \rho_n(x))} \\
&= \frac{\rho_n(x) (\mu'_n(x) - \mu_n(x))}{\mu_n(x) (\rho'_n(x) - \rho_n(x))} \; \frac{1-\rho'_n(x)}{n(\mu'_n(x) - \mu_n(x))} \\
&= \frac{\rho_n(x)}{\mu_n(x)} \; \frac{\mu'_n(x) - \mu_n(x)}{\rho'_n(x) - \rho_n(x)} \; \frac{1-\rho'_n(x)}{1 - \mu'_n(x)} , 
\end{align*}
with the last line following from Lemma \ref{lem:empirical_rate_of_change}. Under Assumption \ref{assum:limit_assumptions}, all terms of the right-hand side converge as $n \to \infty$. Taking the limit on both sides,
\begin{align*}
\lim_{n \to \infty} \frac{\hN_n(x)}{N_n(x)} &\overset{(a)}{=} \frac{r(x)}{\dr(x)} \lim_{n \to \infty} \frac{1-\rho'_n(x)}{1 - \mu'_n(x)} \\ 
&\overset{(b)}{=} \frac{r(x)}{\dr(x)} \; \frac{1 - \mu(x) r(x)}{1 - \mu(x)} ,
\end{align*}
where (a) is justified by the existence of the relevant limits and $\dr(x) > 0$, and (b) follows 
from writing $\rho'_n(x)$ as $\mu_n(x) \rho'_n(x) / \mu_n(x)$, where all limits involved
exist. 
\end{proof}

\subsection{Directed Graphical Models}

We say that $\cX$ is a \emph{factored} state space if it is the Cartesian product of $k$ subspaces, i.e. $\cX := \cX_1 \times \dots \times \cX_k$. This factored structure allows us to construct approximate density models over $\cX$, for example by modelling the joint density as a product of marginals.
We write the $i^{th}$ factor of a state $x \in \cX$ as $x^i$, and write the sequence of the $i^{th}$ factor across $\xn$ as $x^i_{1:n}$. 

We will show that directed graphical models \citep{wainwright08graphical} satisfy Assumption 
\ref{assum:limit_assumptions}.
A directed graphical model describes a probability distribution over a factored state space. To the $i^{th}$ factor $x^i$ is associated a parent set $\pi(i) \subseteq \left \{ 1, \dots, i - 1 \right \}$. Let $x^{\pi(i)}$ denote the value of the factors in the parent set. The $i^{th}$ factor model
is $\rho^i_n(x^i \csemi x^{\pi(i)}) := \rho^i(x^i \csemi \xn, x^{\pi(i)})$, with the understanding 
that $\rho^i$ is allowed to make a different prediction for each value of $x^{\pi(i)}$. The state $x$ 
is assigned the joint probability 
\begin{equation*}
\rho_{\textsc{gm}}(x \csemi \xn) := \prod_{i=1}^k \rho^i_n(x^i \csemi x^{\pi(i)}) .
\end{equation*}
Common choices for $\rho^i_n$ include the conditional empirical distribution and the Dirichlet estimator. 
\begin{prop}\label{prop:directed_graphical_models}
Suppose that each factor model $\rho^i_n$ converges to the conditional probability distribution
$\mu(x^i \cbar x^{\pi(i)})$ and that for each $x^i$ with $\mu(x^i \cbar x^{\pi(i)})$,
\begin{equation*}
\lim_{n \to \infty} \frac{\rho^i(x^i \csemi \xn x, x^{\pi(i)}) - \rho^i(x^i \csemi \xn, x^{\pi(i)})}{\mu(x^i \csemi \xn x, x^{\pi(i)}) - \mu(x^i \csemi \xn, x^{\pi(i)})} = 1 .
\end{equation*}
Then for all $x$ with $\mu(x) > 0$, the density model $\rho_{\textsc{gm}}$ satisfies Assumption \ref{assum:limit_assumptions} with
\begin{equation*}
r(x) = \frac{\prod_{i=1}^k \mu(x^i \cbar x^{\pi(i)})}{\mu(x)} \qquad \text{ and } \qquad \dr(x) = \frac{\sum_{i=1}^k \big (1 - \mu(x^i \cbar x^{\pi(i)}) \big ) \prod_{j \ne i} \mu(x^j \cbar x^{\pi(j)})}{1 - \mu(x)} .
\end{equation*}
\end{prop}
The CTS density model used in our experiments is in fact a particular kind of induced graphical model. The result above thus describes how the pseudo-counts computed in Section \ref{sec:pseudo_counting_salient_events} are asymptotically related to the empirical counts.

\begin{proof}
By hypothesis, $\rho_n^i \to \mu(x^i \cbar x^{\pi(i)})$. Combining this with $\mu_n(x) \to \mu(x) > 0$, 
\begin{align*}
r(x) &= \lim_{n\to\infty} \frac{\rho_\gm(x \csemi \xn)}{\mu_n(x)} \\
&= \lim_{n\to\infty} \frac{\prod_{i=1}^k \rho_n^i (x^i \csemi x^{\pi(i)})}{\mu_n(x)} \\
&= \frac{\prod_{i=1}^k \mu(x^i \cbar x^{\pi(i)})}{\mu(x)} .
\end{align*}
Similarly,
\begin{align*}
\dr(x) &= \lim_{n\to\infty} \frac{\rho'_\gm(x \csemi \xn) - \rho_\gm(x \csemi \xn)}{\mu'_n(x) - \mu_n(x)} \\
&\overset{(a)}{=} \lim_{n\to\infty} \frac{\big ( \rho'_\gm(x \csemi \xn) - \rho_\gm(x \csemi \xn) \big ) n}{1 - \mu'_n(x)} \\
&=  \lim_{n\to\infty} \frac{\big ( \rho'_\gm(x \csemi \xn) - \rho_\gm(x \csemi \xn) \big) n}{1 - \mu(x)},
\end{align*}
where in (a) we used the identity $n (\mu'_n(x) - \mu_n(x)) = 1 - \mu'_n(x)$ derived in the proof
of Theorem \ref{thm:ratio_of_counts}. Now 
\begin{align*}
\dr(x) &= (1 - \mu(x))^{-1} \lim_{n \to \infty} \big ( \rho'_\gm(x \csemi \xn) - \rho_\gm(x \csemi \xn) \big) n \\
&= (1 - \mu(x))^{-1} \lim_{n \to \infty} \big ( \prod_{i=1}^k \rho^i(x^i \csemi \xn x, x^{\pi(i)}) - \prod_{i=1}^k \rho^i(x^i \csemi \xn, x^{\pi(i)}) \big ) n . \\
\end{align*}
Let $c_i := \rho^i(x^i \csemi \xn, x^{\pi(i)})$ and $c'_i := \rho^i(x^i \csemi \xn x, x^{\pi(i)})$. The difference of products above is
\begin{align*}
 \big ( \prod_{i=1}^k \rho^i(x^i \csemi \xn x, x^{\pi(i)}) - \prod_{i=1}^k \rho^i(x^i \csemi \xn, x^{\pi(i)}) \big ) &= \big (c'_1 c'_2 \dots c'_k - c_1 c_2 \dots c_k \big ) \\
 &= (c'_1 - c_1) (c'_2 \dots c'_k) + c_1 (c'_2 \dots c'_k - c_2 \dots c_k) \\
 &= \sum_{i=1}^k (c'_i - c_i) \Big ( \prod_{j < i} c_j \Big ) \Big ( \prod_{j > i} c'_j \Big ),
\end{align*}
and
\begin{equation*}
\dr(x) = (1 - \mu(x))^{-1} \lim_{n \to \infty} \sum_{i=1}^k n (c'_i - c_i) \Big ( \prod_{j < i} c_j \Big ) \Big ( \prod_{j > i} c'_j \Big ) .
\end{equation*}
By the hypothesis on the rate of change of $\rho^i$ and the identity $n \left ( \mu(x^i \csemi \xn x, x^{\pi(i)}) - \mu(x^i \csemi \xn, x^{\pi(i)}) \right ) = 1 - \mu(x^i \cbar x^{\pi(i)})$, we have 
\begin{equation*}
\lim_{n \to \infty} n (c'_i - c_i) = 1 - \mu(x^i \cbar x^{\pi(i)}) .
\end{equation*}
Since the limits of $c'_i$ and $c_i$ are both $\mu(x^i \cbar x^{\pi(i)})$, we deduce that
\begin{equation*}
\dr(x) = \frac{\sum_{i=1}^k \big (1 - \mu(x^i \cbar x^{\pi(i)}\big ) \prod_{j \ne i} \mu(x^j \cbar x^{\pi_j(x)})}{1 - \mu(x)} .
\end{equation*}
Now, if $\mu(x) > 0$ then also $\mu(x^i \csemi x^{\pi(i)}) > 0$ for each factor $x^i$. Hence $\dr(x) > 0$.
\end{proof}

\subsection{Tabular Density Models (Corollary \ref{cor:atomic_models})}
We shall prove the following, which includes Corollary \ref{cor:atomic_models} as a special case. 

\begin{lem}\label{lem:atomic_models_extended}
Consider $\phi : \cX \times \cX^* \to \bR^+$. Suppose that for all $(x_n : n \in \bN)$ and every $x \in \cX$
\begin{enumerate}
    \item{$\lim\limits_{n \to \infty} \tfrac{1}{n} \sum\limits_{x \in \cX} \phi(x, \xn) = 0$, and}
    \item{$\lim\limits_{n \to \infty} \big (\phi(x, \xn x) - \phi(x, \xn) \big ) = 0$.}
\end{enumerate}
Let $\rho_n(x)$ be the count-based estimator
\begin{equation*}
\rho_n(x) = \frac{N_n(x) + \phi(x, \xn)}{n + \sum_{x \in \cX} \phi(x, \xn)}.
\end{equation*}
If $\hN_n$ is the pseudo-count corresponding to $\rho_n$ then $\hN_n(x) / N_n(x) \to 1$ for all $x$ with $\mu(x) > 0$. 
\end{lem}
Condition 2 is satisfied if $\phi_n(x, \xn) = u_n(x) \phi_n$ with $\phi_n$ monotonically increasing in $n$ (but not too quickly!) and $u_n(x)$ converging to some distribution $u(x)$ for all sequences $(x_n : n \in \bN)$. This is the case for most tabular density models. 

\begin{proof}
We will show that the condition on the rate of change required by Proposition \ref{prop:directed_graphical_models} is satisfied under the stated conditions. 
Let $\phi_n(x) := \phi(x, \xn)$, $\phi'_n(x) := \phi(x, \xn x)$, $\phi_n := \sum_{x \in \cX} \phi_n(x)$ and $\phi'_n := \sum_{x \in \cX} \phi'_n(x)$.
By hypothesis,
\begin{equation*}
\rho_n(x) = \frac{N_n(x) + \phi_n(x)}{n + \phi_n} \qquad \qquad \rho'_n(x) = \frac{N_n(x) + \phi'_n(x) + 1}{n + \phi'_n + 1}.
\end{equation*}
Note that we do not require $\phi_n(x) = \phi'_n(x)$. Now
\begin{align*}
\rho'_n(x) - \rho_n(x) &= \frac{n + \phi_n}{n + \phi_n} \rho'_n(x) - \rho_n(x) \\
&= \frac{n + 1 + \phi'_n}{n + \phi_n} \rho'_n(x) - \rho_n(x) - \frac{(1 + (\phi'_n - \phi_n)) \rho'_n(x)}{n + \phi_n} \\
&= \frac{1}{n + \phi_n} \Big [ (N_n(x) + 1 + \phi'_n(x) - (N_n(x) + \phi_n(x)) - (1 + (\phi'_n - \phi_n)) \rho'_n(x) \Big ] \\
&= \frac{1}{n + \phi_n} \Big [ 1 - \rho'_n(x) + \big ( \phi'_n(x) - \phi_n(x) \big ) - \rho'_n(x) \big ( \phi'_n - \phi_n \big ) \Big ] .
\end{align*}
Using Lemma \ref{lem:empirical_rate_of_change} we deduce that
\begin{equation*}
\frac{\rho'_n(x) - \rho_n(x)}{\mu'_n(x) - \mu_n(x)} = \frac{n}{n + \phi_n} \; \frac{1 - \rho'_n(x) + \phi'_n(x) - \phi_n(x) + \rho'_n(x) ( \phi'_n - \phi_n )}{1 - \mu'_n(x)} . 
\end{equation*}
Since $\phi_n = \sum_x \phi_n(x)$ and similarly for $\phi'_n$, then $\phi'_n(x) - \phi_n(x) \to 0$ pointwise implies that $\phi'_n - \phi_n \to 0$ also. For any $\mu(x) > 0$,
\begin{align*}
0 \le \lim_{n \to \infty} \frac{\phi_n(x)}{N_n(x)} &\overset{(a)}{\le} \lim_{n \to \infty} \frac{\sum_{x \in \cX} \phi_n(x)}{N_n(x)} \\
&= \lim_{n \to \infty} \frac{\sum_{x \in \cX} \phi_n(x)}{n} \frac{n}{N_n(x)} \\
&\overset{(b)}{=} 0,
\end{align*}
where a) follows from $\phi_n(x) \ge 0$ and b) is justified by $n / N_n(x) \to \mu(x)^{-1} > 0$ and the hypothesis that $\sum_{x \in \cX} \phi_n(x) / n \to 0$.
Therefore $\rho_n(x) \to \mu(x)$. Hence 
\begin{equation*}
\lim_{n \to \infty} \frac{\rho'_n(x) - \rho_n(x)}{\mu'_n(x) - \mu_n(x)} = \lim_{n \to \infty} \frac{n}{n + \phi_n} \; \frac{1 - \rho'_n(x)}{1 - \mu'_n(x)} = 1. 
\end{equation*}
Since $\rho_n(x) \to \mu(x)$, we further deduce from Theorem \ref{thm:ratio_of_counts} that
\begin{equation*}
\lim_{n \to \infty} \frac{\hN_n(x)}{N_n(x)} = 1.\qedhere
\end{equation*}
\end{proof}
The condition $\mu(x) > 0$, which was also needed in Proposition \ref{prop:directed_graphical_models}, is necessary for the ratio to converge to 1: for example, if $N_n(x)$ grows as $O(\log n)$
but $\phi_n(x)$ grows as $O(\sqrt{n})$ (with $|\cX|$ finite) then $\hN_n(x)$ will grow as the 
larger $\sqrt{n}$. 

\section{Experimental Methods}

\subsection{CTS Density Model}

Our state space $\cX$ is the set of all preprocessed Atari 2600 frames.\footnote{Technically, the ALE is partially observable and a frame is an observation, not a state. In many games, however, the current frame is sufficiently informative to guide exploration.} Each raw frame is composed of
$210 \times 160$ 7-bit NTSC pixels \citep{bellemare13arcade}. We preprocess these frames by
first converting them to grayscale (luminance), then downsampling to $42 \times 42$ by averaging
over pixel values (Figure \ref{fig:preprocessing}).

\begin{figure*}
\center{
\includegraphics[height=2in]{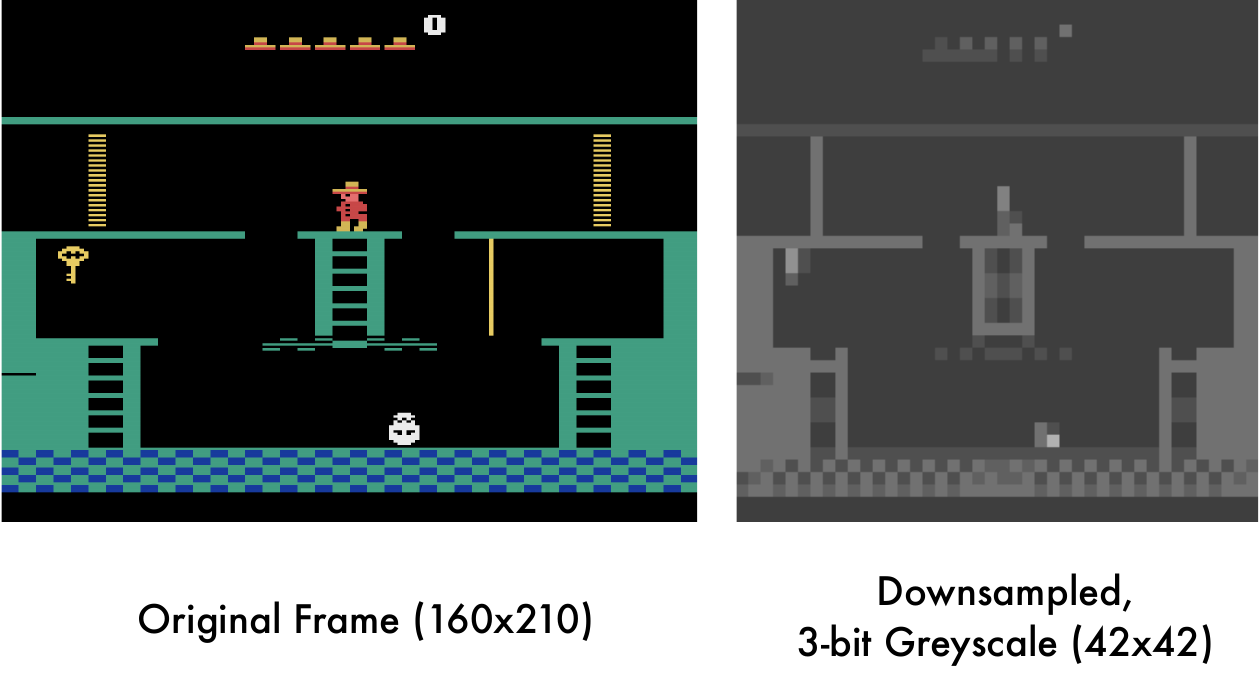}
}
\caption{Sample preprocessed image provided to the CTS model (\textbf{right}), along with the original frame (\textbf{left}). Although details are lost, objects can still be made out.\label{fig:preprocessing}}
\end{figure*}

Aside from this preprocessing, our model is very similar to the model used by 
\citet{bellemare14skip} and \citet{veness15compress}.
The CTS density model treats $x \in \cX$ as a factored state, where each $(i,j)$
pixel corresponds to a factor $x^{i,j}$. The parents of this factor are its 
upper-left neighbours, i.e. pixels $(i-1,j)$, $(i,j-1)$, $(i-1,j-1)$ and $(i+1,j-1)$ (in this order).
The probability of $x$ is then the product of the probability assigned to its factors. Each
factor is modelled using a location-dependent CTS model, which predicts the pixel's colour value
conditional on some, all, or possibly none, of the pixel's parents (Figure \ref{fig:filter}).

\begin{figure*}
\center{
\includegraphics{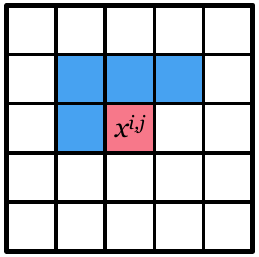}
}
\caption{Depiction of the CTS ``filter''. Each downsampled pixel is predicted by a location-specific model which can condition on the pixel's immediate neighbours (in blue).\label{fig:filter}}
\end{figure*}

\subsection{A Taxonomy of Exploration}

We provide in Table \ref{table:taxonomy} a rough taxonomy of the Atari 2600 games available through the ALE in terms of the difficulty of exploration.

We first divided the games into two groups: those for which local exploration (e.g. $\epsilon$-greedy) is sufficient to achieve a high scoring policy (\emph{easy}), and those for which it is not (\emph{hard}). For example, \gamename{Space Invaders} versus \gamename{Pitfall!}. We further divided the \emph{easy} group based on whether an $\epsilon$-greedy scheme finds a \emph{score exploit}, that is maximizes the score without achieving the game's stated objective. For example, \gamename{Kung-Fu Master} versus \gamename{Boxing}. While this distinction is not directly used here, score exploits lead to behaviours which are optimal from an ALE perspective but uninteresting to humans. We divide the games in the \emph{hard} category into dense reward games (\gamename{Ms. Pac-Man}) and sparse reward games (\gamename{Montezuma's Revenge}).

\begin{table}
\center{
\scriptsize{
\begin{tabular}{|cc|c||c|c|}
\hline
\multicolumn{3}{c}{\normalsize{Easy Exploration}} & \multicolumn{2}{c}{\normalsize{Hard Exploration}} \\
\hline
\hline
\multicolumn{2}{c}{\small{Human-Optimal}} & \multicolumn{1}{c}{\small{Score Exploit}} & \multicolumn{1}{c}{\small{Dense Reward}} & \multicolumn{1}{c}{\small{Sparse Reward}} \\
\hline
\gamename{Assault } &\gamename{ Asterix } &\gamename{ Beam Rider } &\gamename{ Alien } &\gamename{ Freeway } \\
\hline
\gamename{Asteroids } &\gamename{ Atlantis } &\gamename{ Kangaroo } &\gamename{ Amidar } &\gamename{ Gravitar } \\
\hline
\gamename{Battle Zone } &\gamename{ Berzerk } &\gamename{ Krull } &\gamename{ Bank Heist } &\gamename{ Montezuma's Revenge } \\
\hline
\gamename{Bowling } &\gamename{ Boxing } &\gamename{ Kung-fu Master } &\gamename{ Frostbite } &\gamename{ Pitfall! } \\
\hline
\gamename{Breakout } &\gamename{ Centipede } &\gamename{ Road Runner } &\gamename{ H.E.R.O. } &\gamename{ Private Eye } \\
\hline
\gamename{Chopper Cmd } &\gamename{ Crazy Climber } &\gamename{ Seaquest } &\gamename{ Ms. Pac-Man } &\gamename{ Solaris } \\
\hline
\gamename{Defender } &\gamename{ Demon Attack } &\gamename{ Up n Down } & \gamename{Q*Bert} & \gamename{Venture} \\
\hline
\gamename{Double Dunk } &\gamename{ Enduro } & \gamename{ Tutankham } &\gamename{ Surround } & \\
\hline
\gamename{Fishing Derby } &\gamename{ Gopher } &  &\gamename{ Wizard of Wor } & \\
\hline
\gamename{Ice Hockey } &\gamename{ James Bond } &  &\gamename{ Zaxxon } & \\
\hline
\gamename{Name this Game } &\gamename{ Phoenix } &  & & \\
\hline
\gamename{Pong } &\gamename{ River Raid } & &  & \\
\hline
\gamename{Robotank } &\gamename{ Skiing } &  &  & \\
\hline
\gamename{Space Invaders } &\gamename{ Stargunner } &  &  & \\
\hline
\end{tabular}
}}
\caption{A rough taxonomy of Atari 2600 games according to their exploration difficulty.\label{table:taxonomy}}
\end{table}

\subsection{Exploration in \gamename{Montezuma's Revenge}}

\textsc{Montezuma's Revenge} is divided into three levels, each composed of 24 rooms arranged in
a pyramidal shape (Figure \ref{fig:montezuma_revenge}). As discussed above, each room poses
a number of challenges: to escape the very first room, the agent must climb ladders, dodge a 
creature, pick up a key, then backtrack to open one of two doors. The number of rooms reached by
an agent is therefore a good measure of its ability. 
\begin{figure*}
\center{
\includegraphics[width=5in]{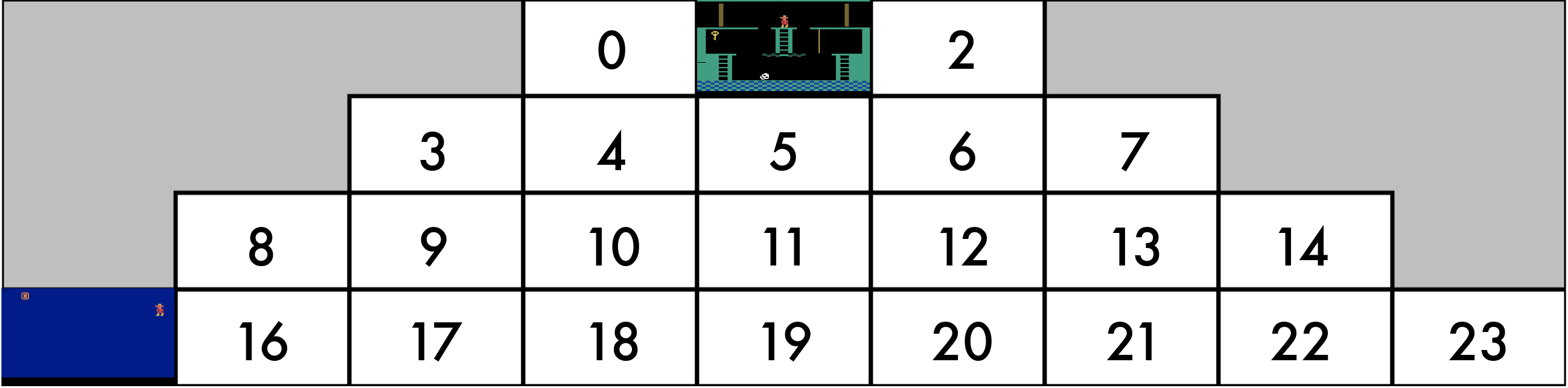}
}
\caption{Layout of levels in \gamename{Montezuma's Revenge}, with rooms numbered from 0 to 23.
The agent begins in room 1 and completes the level upon reaching room 15 (depicted).\label{fig:montezuma_revenge}}
\end{figure*}
By accessing the game RAM, we recorded the location of the agent at each step during the course of
training.\footnote{We emphasize that the game RAM is not made available to the agent, and is solely
used here in our behavioural analysis.} We computed the visit count to each room, averaged over
epochs each lasting one million frames. From this information we constructed a map of the agent's
``known world'', that is, all rooms visited at least once. 
The agent's current room number ranges from 0 to 23 (Figure \ref{fig:montezuma_revenge}) and is
stored at RAM location 0x83. Figure \ref{fig:montezuma_revenge_all_maps} shows the set of rooms
explored by our DQN agents at different points during training. 

Figure \ref{fig:montezuma_revenge_all_maps} paints a clear picture: after 50 million frames, the
agent using exploration bonuses has seen a total of 15 rooms, while the no-bonus agent has seen 
two. At that point in time, our agent achieves an average score of \textbf{2461}; by 100 million
frames, this figure stands at \textbf{3439}, higher than anything previously reported.
We believe the success of our method in this game is a strong indicator of the 
usefulness of pseudo-counts for exploration.

We remark that without mixing in the Monte-Carlo return, our bonus-based agent still explores
significantly more than the no-bonus agent. However, the deep network seems unable to maintain
a sufficiently good approximation to the value function, and performance quickly deteriorates.
Comparable results using the A3C method provide another example of the practical importance of
eligibility traces and return-based methods in reinforcement learning.

\begin{figure*}
\center{
\includegraphics[width=\textwidth,clip]{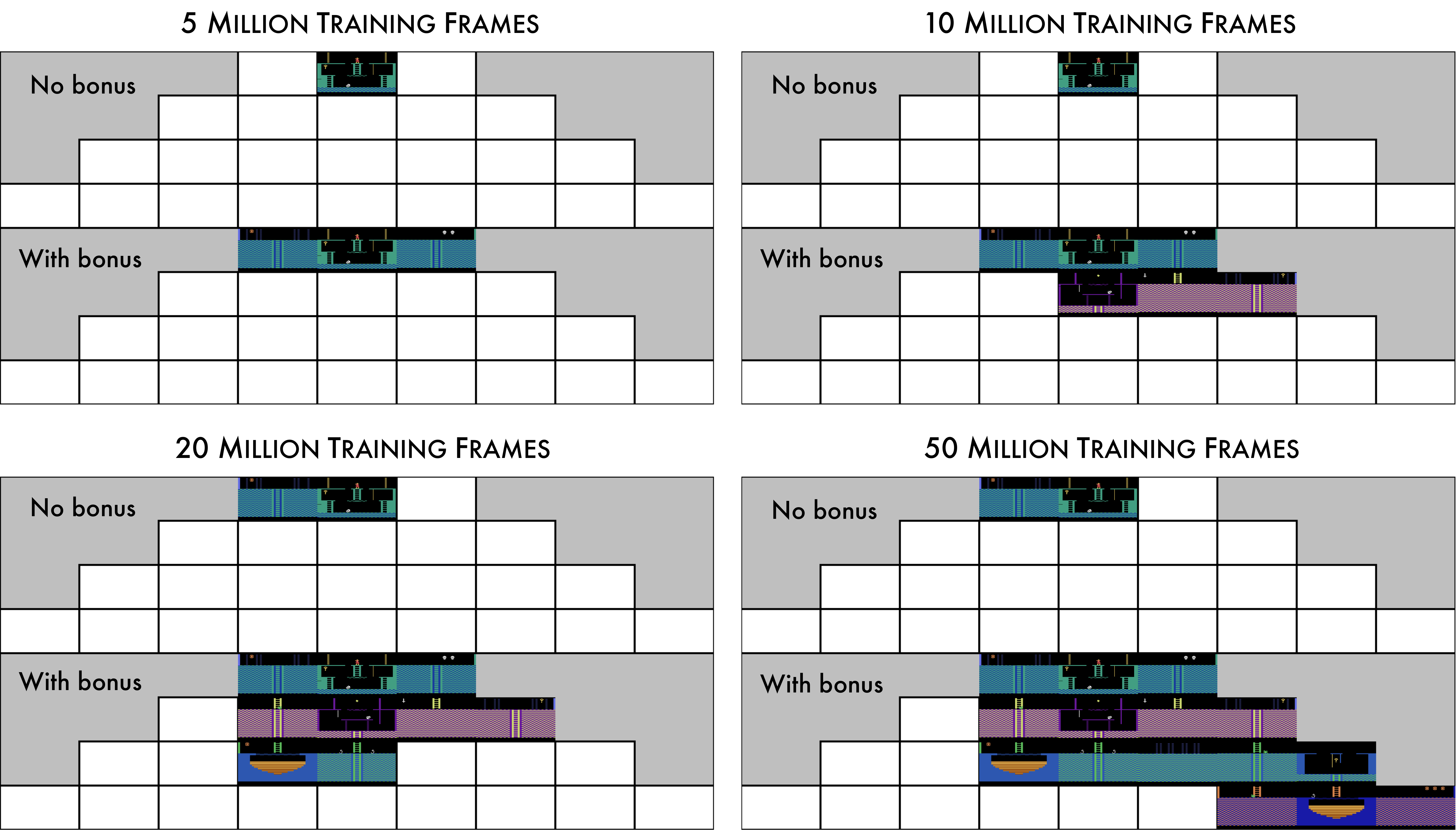}
}
\caption{``Known world'' of a DQN agent trained over time, with (\textbf{bottom}) and without (\textbf{top}) count-based exploration bonuses, in \gamename{Montezuma's Revenge}.\label{fig:montezuma_revenge_all_maps}}
\end{figure*}

\subsection{Improving Exploration for Actor-Critic Methods}

Our implementation of A3C was along the lines mentioned in \cite{mnih16asynchronous} and uses 16 threads. Each thread corresponds to an actor learner and maintains a copy of the density model. All the threads are synchronized with the master thread at regular intervals of 250,000 steps. We followed the same training procedure as that reported in the A3C paper with the following additional steps:
We update our density model with the states generated by following the policy. During the policy gradient step, we compute the intrinsic rewards by querying the density model and add it to the extrinsic rewards before clipping them in the range $[-1, 1]$ as was done in the A3C paper. This resulted in minimal overhead in computation costs and the memory footprint was manageable ($<$ 32 GB) for most of the Atari games. Our training times were almost the same as the ones reported in the A3C paper.
We picked $\beta = 0.01$ after performing a short parameter sweep over the training games. The choice of training games is the same as mentioned in the A3C paper.

The games on which DQN achieves a score of 150\% or less of the random score are: \gamename{Asteroids}, \gamename{Double Dunk}, \gamename{Gravitar}, \gamename{Ice Hockey}, \gamename{Montezuma's Revenge}, \gamename{Pitfall!}, \gamename{Skiing}, \gamename{Surround}, \gamename{Tennis}, \gamename{Time Pilot}.

The games on which A3C achieves a score of 150\% or less of the random score are: \gamename{Battle Zone}, \gamename{Bowling}, \gamename{Enduro}, \gamename{Freeway}, \gamename{Gravitar}, \gamename{Kangaroo}, \gamename{Pitfall!}, \gamename{Robotank}, \gamename{Skiing}, \gamename{Solaris}, \gamename{Surround}, \gamename{Tennis}, \gamename{Time Pilot}, \gamename{Venture}.

The games on which A3C+ achieves a score of 150\% or less of the random score are: \gamename{Double Dunk}, \gamename{Gravitar}, \gamename{Ice Hockey}, \gamename{Pitfall!}, \gamename{Skiing}, \gamename{Solaris}, \gamename{Surround}, \gamename{Tennis}, \gamename{Time Pilot}, \gamename{Venture}.

Our experiments involved the stochastic version of the Arcade Learning Environment (ALE) without a terminal signal for life loss, which is now the default ALE setting. Briefly, the stochasticity is achieved by accepting the agent’ action at each frame with probability $1 - p$ and using the agent’s previous action during rejection. We used the ALE's default value of $p = 0.25$ as has been previously used in \cite{bellemare16increasing}. For comparison, Table \ref{table:all_the_a3c_results} also reports the deterministic + life loss setting also used in the literature.

Anecdotally, we found that using the life loss signal, while helpful in achieving high scores in
some games, is detrimental in \gamename{Montezuma's Revenge}. Recall that the life loss signal
was used by \citet{mnih15human} to treat each of the agent' lives as a separate episode. 
For comparison, after 200 million
frames A3C+ achieves the following average scores: 1) Stochastic + Life Loss: 142.50; 2) Deterministic + Life Loss: 273.70 3) Stochastic without Life Loss: 1127.05 4) Deterministic without Life Loss: 273.70. The maximum score achieved by 3) is 3600, in comparison to the maximum of 500 achieved by 1) and 3). This large discrepancy is not unsurprising when one considers that losing a life in
\gamename{Montezuma's Revenge}, and in fact in most games, is very different from restarting a
new episode.

\begin{figure*}
\center{
\includegraphics[width=\textwidth,trim={0.25cm 0.25cm 0.25cm 0.25cm},clip]{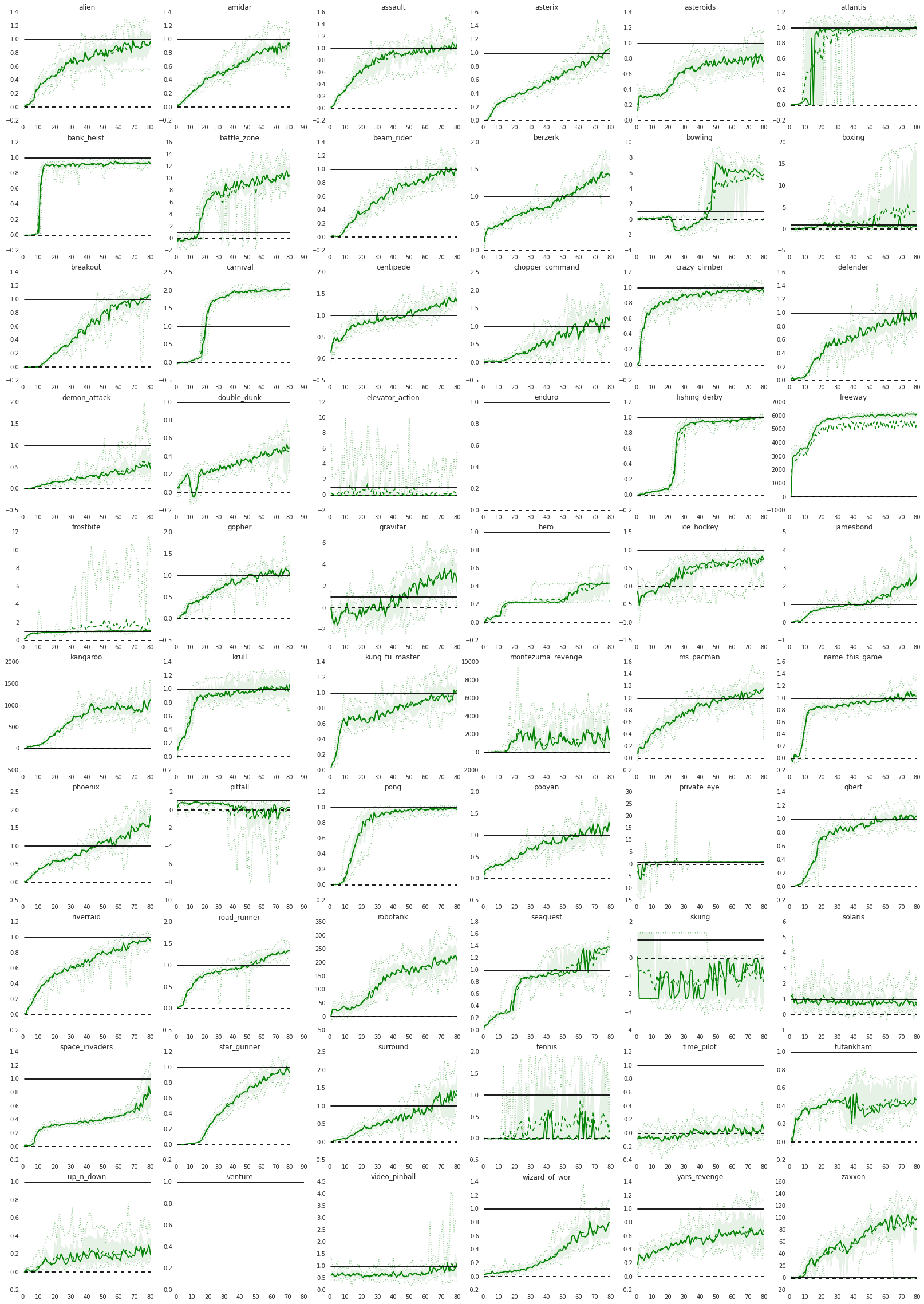}
}
\caption{Average A3C+ score (solid line) over 200 million training frames, for all Atari 2600 games, normalized relative to the A3C baseline. 
Dotted lines denote min/max over seeds, inter-quartile range is shaded, and the median is dashed.
\label{fig:a3c_all}} 
\end{figure*}

\subsection{Comparing Exploration Bonuses}

In this section we compare the effect of using different exploration bonuses derived from our density model. We consider the following variants:
\begin{itemize}
    \item{no exploration bonus,}
    \item{$\hN_n(x)^{-1/2}$, as per MBIE-EB \citep{strehl08analysis};}
    \item{$\hN_n(x)^{-1}$, as per BEB \citep{kolter09near}; and}
    \item{$\PG_n(x)$, related to compression progress \citep{schmidhuber08driven}.}
\end{itemize}
The exact form of these bonuses is analogous to \eqnref{square_root_bonus}.
We compare these variants after 10, 50, 100, and 200 million frames of training, again in the A3C setup. To compare scores across 60 games, we use inter-algorithm score distributions \citep{bellemare13arcade}. Inter-algorithm scores are normalized so that 0 corresponds to the worst score on a game, and 1, to the best. 
If $g \in \{ 1, \dots, m \}$ is a game and $z_{g,a}$ the inter-algorithm score on $g$ for algorithm $a$, then the score distribution function is 
\begin{equation*}
f(x) := \frac{| \{ g : z_{g,a} \ge x \} |}{m} .
\end{equation*}
The score distribution effectively depicts a kind of cumulative distribution, with a higher overall curve implying better scores across the gamut of Atari 2600 games. A higher curve at $x = 1$ implies top performance on more games; a higher curve at $x = 0$ indicates the algorithm does not perform poorly on many games. The scale parameter $\beta$ was optimized to $\beta = 0.01$ for each variant separately. 

Figure \ref{fig:exploration_bonuses_over_time} shows that, while prediction gain initially achieves strong performance, by 50 million frames all three algorithms perform equally well. By 200 million frames, the $\hN^{-1/2}$ exploration bonus outperforms both prediction gain and no bonus. The prediction gain achieves a decent, but not top-performing score on all games. This matches our earlier argument that using prediction gain results in too little exploration. We hypothesize that the poor performance of the $\hN^{-1}$ bonus stems from too abrupt a decay from a large to small intrinsic reward, although more experiments are needed. 
As a whole, these results show how using PG offers an advantage over the baseline A3C algorithm, which is furthered by using our count-based exploration bonus.
\begin{figure*}
\center{
\includegraphics[width=2.5in]{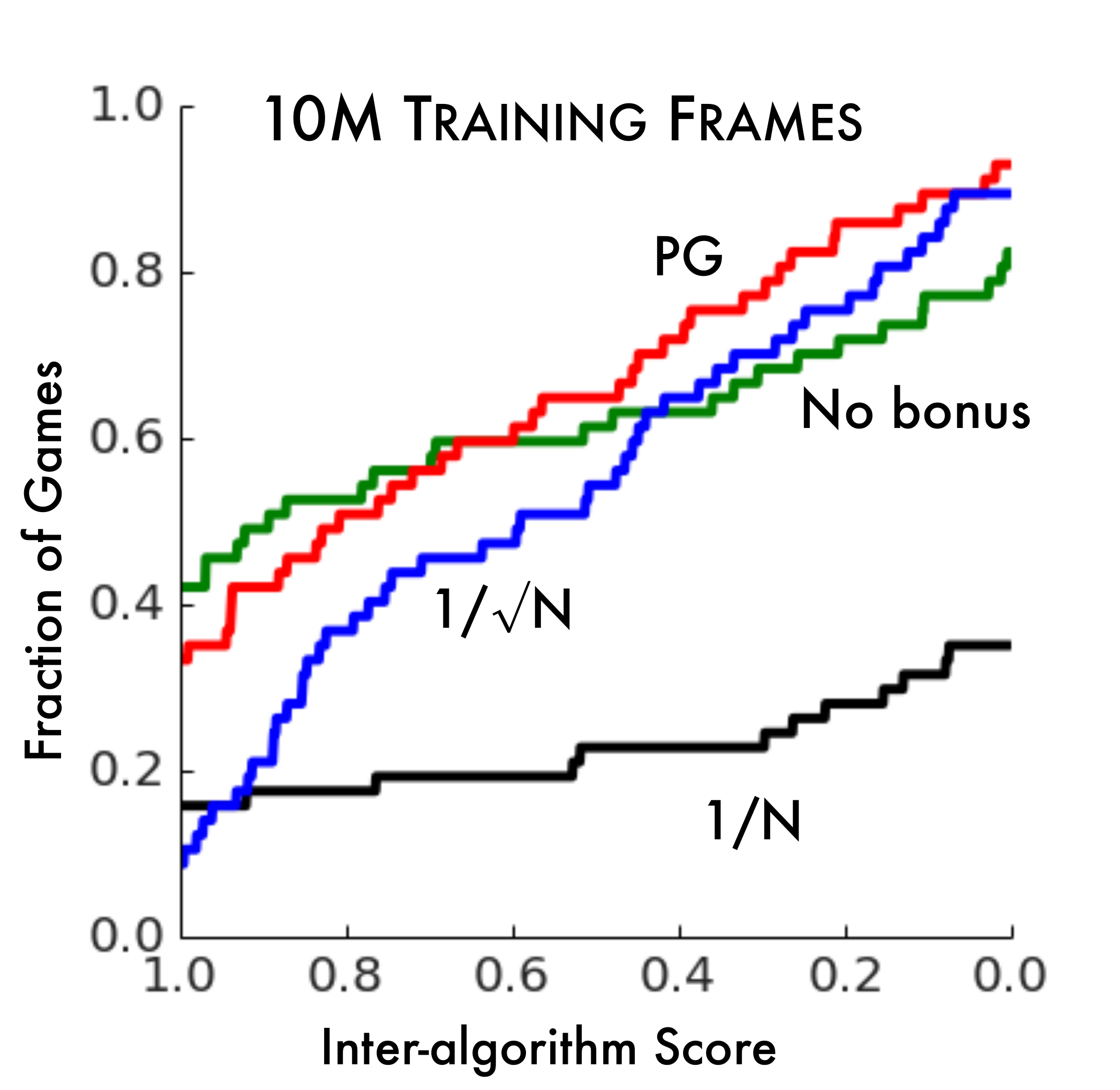}
\includegraphics[width=2.5in]{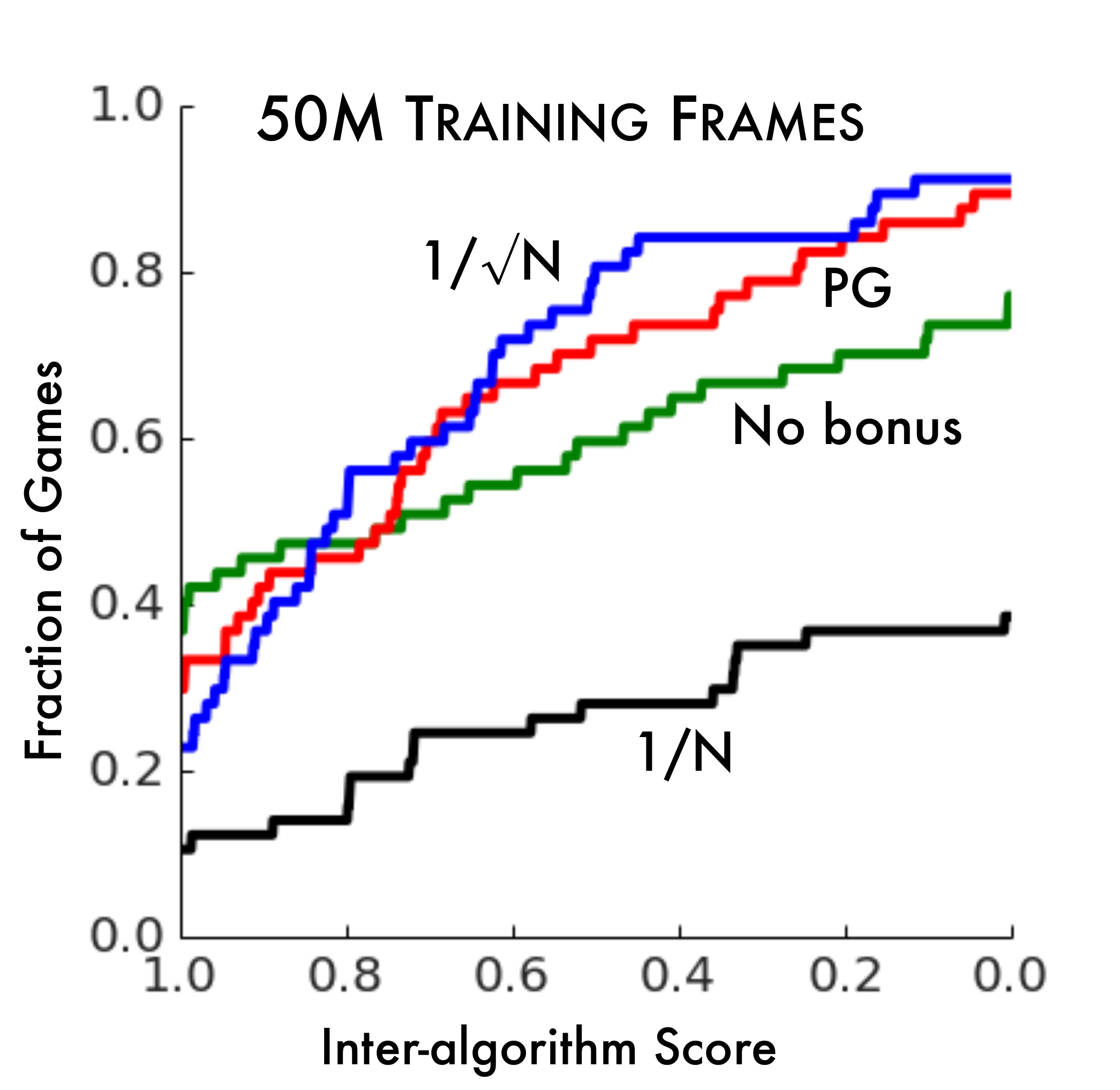}
\includegraphics[width=2.5in]{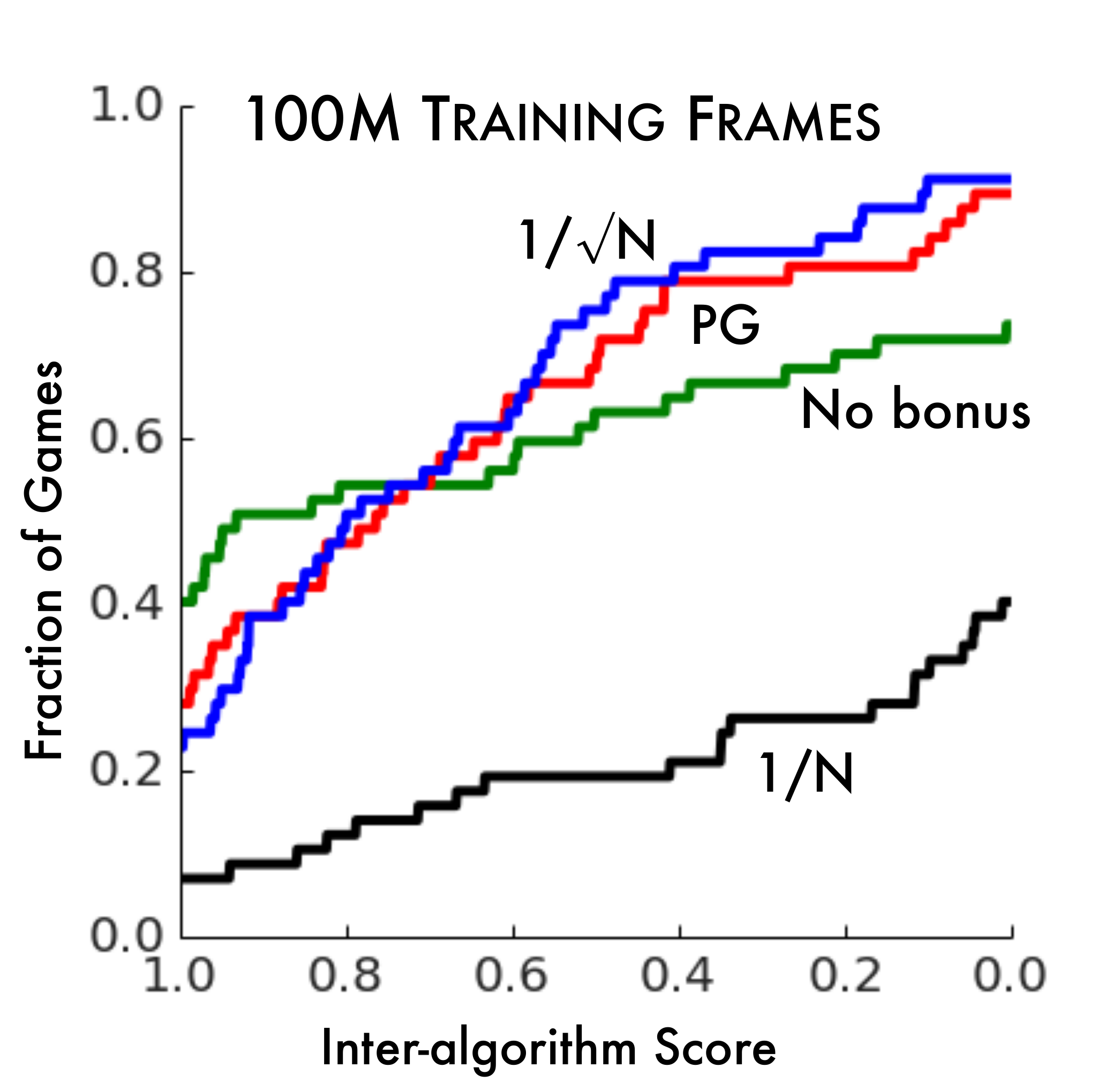}
\includegraphics[width=2.5in]{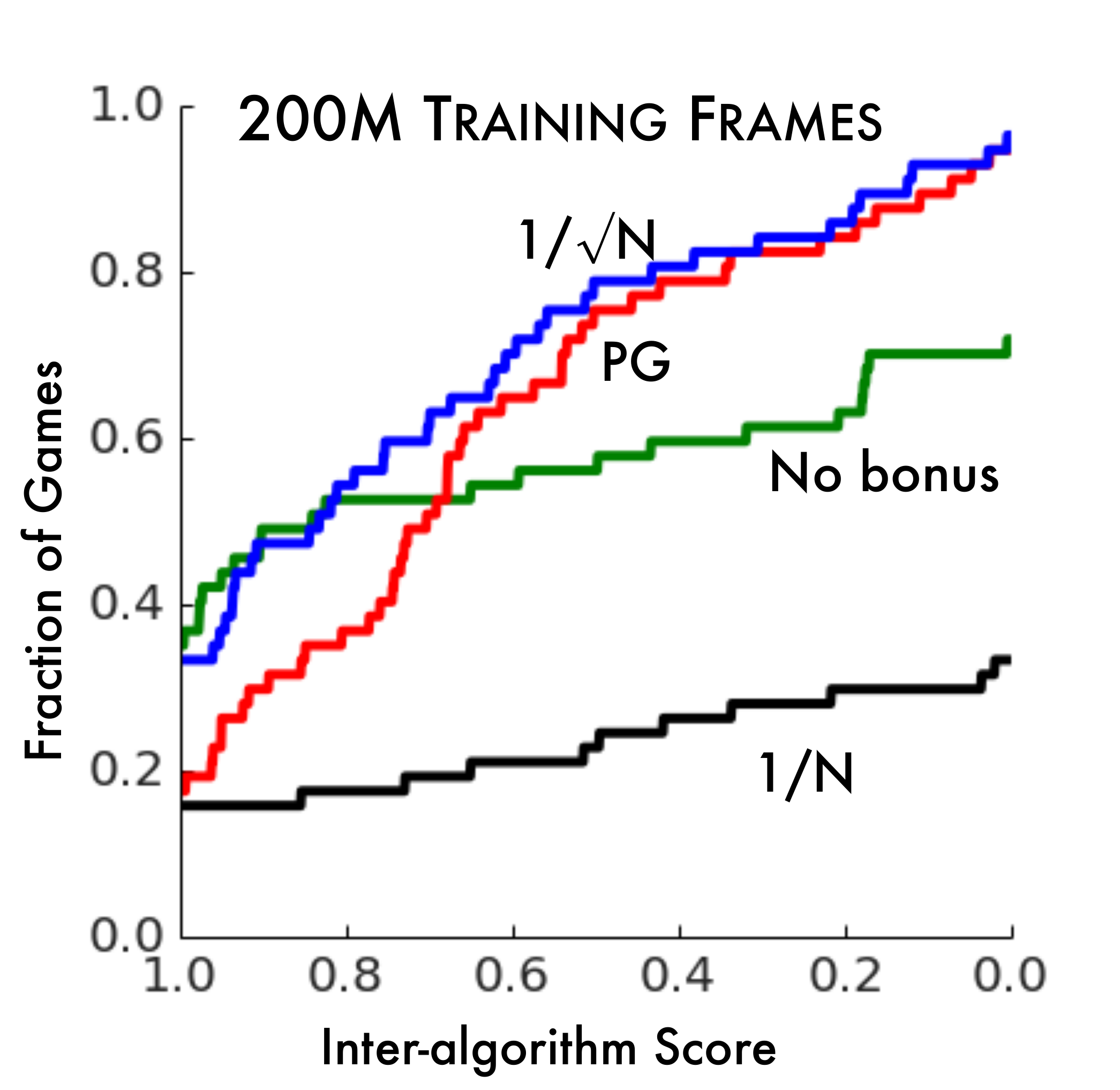}
}
\caption{Inter-algorithm score distribution for exploration bonus variants. 
For all methods the point $f(0) = 1$ is omitted for clarity. See text for details.\label{fig:exploration_bonuses_over_time}} 
\end{figure*}

\begin{table}
\center{
\small
\begin{tabular}{|r|r|r|r||r|r|r|}
\multicolumn{1}{c}{} & \multicolumn{3}{c}{Stochastic ALE} & \multicolumn{3}{c}{Deterministic ALE} \\
\hline
& A3C & A3C+ & DQN & A3C & A3C+ & DQN \\
\hline

\textsc{Alien} & \textbf{\textcolor{blue}{1968.40}} & 1848.33 & 1802.08  & 1658.25 & \textbf{\textcolor{black}{1945.66}} & 1418.47  \\
\hline
\textsc{Amidar} & \textbf{\textcolor{blue}{1065.24}} & 964.77 & 781.76  & \textbf{\textcolor{black}{1034.15}} & 861.14 & 654.40  \\
\hline
\textsc{Assault} & \textbf{\textcolor{blue}{2660.55}} & 2607.28 & 1246.83  & \textbf{\textcolor{black}{2881.69}} & 2584.40 & 1707.87  \\
\hline
\textsc{Asterix} & 7212.45 & \textbf{\textcolor{blue}{7262.77}} & 3256.07  & \textbf{\textcolor{black}{9546.96}} & 7922.70 & 4062.55  \\
\hline
\textsc{Asteroids} & \textbf{\textcolor{blue}{2680.72}} & 2257.92 & 525.09  & \textbf{\textcolor{black}{3946.22}} & 2406.57 & 735.05  \\
\hline
\textsc{Atlantis} & \textbf{\textcolor{blue}{1752259.74}} & 1733528.71 & 77670.03  & 1634837.98 & \textbf{\textcolor{black}{1801392.35}} & 281448.80  \\
\hline
\textsc{Bank Heist} & \textbf{\textcolor{blue}{1071.89}} & 991.96 & 419.50  & \textbf{\textcolor{black}{1301.51}} & 1182.89 & 315.93  \\
\hline
\textsc{Battle Zone} & 3142.95 & 7428.99 & \textbf{\textcolor{blue}{16757.88}}  & 3393.84 & 7969.06 & \textbf{\textcolor{black}{17927.46}}  \\
\hline
\textsc{Beam Rider} & \textbf{\textcolor{blue}{6129.51}} & 5992.08 & 4653.24  & 7004.58 & 6723.89 & \textbf{\textcolor{black}{7949.08}}  \\
\hline
\textsc{Berzerk} & 1203.09 & \textbf{\textcolor{blue}{1720.56}} & 416.03  & 1233.47 & \textbf{\textcolor{black}{1863.60}} & 471.76  \\
\hline
\textsc{Bowling} & 32.91 & \textbf{\textcolor{blue}{68.72}} & 29.07  & 35.00 & \textbf{\textcolor{black}{75.97}} & 30.34  \\
\hline
\textsc{Boxing} & 4.48 & 13.82 & \textbf{\textcolor{blue}{66.13}}  & 3.07 & 15.75 & \textbf{\textcolor{black}{80.17}}  \\
\hline
\textsc{Breakout} & 322.04 & \textbf{\textcolor{blue}{323.21}} & 85.82  & 432.42 & \textbf{\textcolor{black}{473.93}} & 259.40  \\
\hline
\textsc{Centipede} & 4488.43 & \textbf{\textcolor{blue}{5338.24}} & 4698.76  & 5184.76 & \textbf{\textcolor{black}{5442.94}} & 1184.46  \\
\hline
\textsc{Chopper Command} & 4377.91 & \textbf{\textcolor{blue}{5388.22}} & 1927.50  & 3324.24 & \textbf{\textcolor{black}{5088.17}} & 1569.84  \\
\hline
\textsc{Crazy Climber} & \textbf{\textcolor{blue}{108896.28}} & 104083.51 & 86126.17  & 111493.76 & \textbf{\textcolor{black}{112885.03}} & 102736.12  \\
\hline
\textsc{Defender} & \textbf{\textcolor{blue}{42147.48}} & 36377.60 & 4593.79  & \textbf{\textcolor{black}{39388.08}} & 38976.66 & 6225.82  \\
\hline
\textsc{Demon Attack} & \textbf{\textcolor{blue}{26803.86}} & 19589.95 & 4831.12  & \textbf{\textcolor{black}{39293.17}} & 30930.33 & 6183.58  \\
\hline
\textsc{Double Dunk} & \textbf{\textcolor{blue}{0.53}} & -8.88 & -11.57  & \textbf{\textcolor{black}{0.19}} & -7.84 & -13.99  \\
\hline
\textsc{Enduro} & 0.00 & \textbf{\textcolor{blue}{749.11}} & 348.30  & 0.00 & \textbf{\textcolor{black}{694.83}} & 441.24  \\
\hline
\textsc{Fishing Derby} & \textbf{\textcolor{blue}{30.42}} & 29.46 & -27.83  & \textbf{\textcolor{black}{32.00}} & 31.11 & -8.68  \\
\hline
\textsc{Freeway} & 0.00 & 27.33 & \textbf{\textcolor{blue}{30.59}}  & 0.00 & \textbf{\textcolor{black}{30.48}} & 30.12  \\
\hline
\textsc{Frostbite} & 290.02 & 506.61 & \textbf{\textcolor{blue}{707.41}}  & 283.99 & 325.42 & \textbf{\textcolor{black}{506.10}}  \\
\hline
\textsc{Gopher} & 5724.01 & \textbf{\textcolor{blue}{5948.40}} & 3946.13  & \textbf{\textcolor{black}{6872.60}} & 6611.28 & 4946.39  \\
\hline
\textsc{Gravitar} & 204.65 & \textbf{\textcolor{blue}{246.02}} & 43.04  & 201.29 & \textbf{\textcolor{black}{238.68}} & 219.39  \\
\hline
\textsc{H.E.R.O.} & \textbf{\textcolor{blue}{32612.96}} & 15077.42 & 12140.76  & \textbf{\textcolor{black}{34880.51}} & 15210.62 & 11419.16  \\
\hline
\textsc{Ice Hockey} & \textbf{\textcolor{blue}{-5.22}} & -7.05 & -9.78  & \textbf{\textcolor{black}{-5.13}} & -6.45 & -10.34  \\
\hline
\textsc{James Bond} & 424.11 & \textbf{\textcolor{blue}{1024.16}} & 511.76  & 422.42 & \textbf{\textcolor{black}{1001.19}} & 465.76  \\
\hline
\textsc{Kangaroo} & 47.19 & \textbf{\textcolor{blue}{5475.73}} & 4170.09  & 46.63 & 4883.53 & \textbf{\textcolor{black}{5972.64}}  \\
\hline
\textsc{Krull} & 7263.37 & \textbf{\textcolor{blue}{7587.58}} & 5775.23  & 7603.84 & \textbf{\textcolor{black}{8605.27}} & 6140.24  \\
\hline
\textsc{Kung-Fu Master} & \textbf{\textcolor{blue}{26878.72}} & 26593.67 & 15125.08  & \textbf{\textcolor{black}{29369.90}} & 28615.43 & 11187.13  \\
\hline
\textsc{Montezuma's Revenge} & 0.06 & \textbf{\textcolor{blue}{142.50}} & 0.02  & 0.17 & \textbf{\textcolor{black}{273.70}} & 0.00  \\
\hline
\textsc{Ms. Pac-Man} & 2163.43 & 2380.58 & \textbf{\textcolor{blue}{2480.39}}  & 2327.80 & \textbf{\textcolor{black}{2401.04}} & 2391.89  \\
\hline
\textsc{Name This Game} & 6202.67 & \textbf{\textcolor{blue}{6427.51}} & 3631.90  & 6087.31 & \textbf{\textcolor{black}{7021.30}} & 6565.41  \\
\hline
\textsc{Phoenix} & 12169.75 & \textbf{\textcolor{blue}{20300.72}} & 3015.64  & 13893.06 & \textbf{\textcolor{black}{23818.47}} & 7835.20  \\
\hline
\textsc{Pitfall} & \textbf{\textcolor{blue}{-8.83}} & -155.97 & -84.40  & \textbf{\textcolor{black}{-6.98}} & -259.09 & -86.85  \\
\hline
\textsc{Pooyan} & 3706.93 & \textbf{\textcolor{blue}{3943.37}} & 2817.36  & 4198.61 & \textbf{\textcolor{black}{4305.57}} & 2992.56  \\
\hline
\textsc{Pong} & \textbf{\textcolor{blue}{18.21}} & 17.33 & 15.10  & \textbf{\textcolor{black}{20.84}} & 20.75 & 19.17  \\
\hline
\textsc{Private Eye} & 94.87 & \textbf{\textcolor{blue}{100.00}} & 69.53  & 97.36 & \textbf{\textcolor{black}{99.32}} & -12.86  \\
\hline
\textsc{Q*Bert} & 15007.55 & \textbf{\textcolor{blue}{15804.72}} & 5259.18  & 19175.72 & \textbf{\textcolor{black}{19257.55}} & 7094.91  \\
\hline
\textsc{River Raid} & \textbf{\textcolor{blue}{10559.82}} & 10331.56 & 8934.68  & \textbf{\textcolor{black}{11902.24}} & 10712.54 & 2365.18  \\
\hline
\textsc{Road Runner} & 36933.62 & \textbf{\textcolor{blue}{49029.74}} & 31613.83  & 41059.12 & \textbf{\textcolor{black}{50645.74}} & 24933.39  \\
\hline
\textsc{Robotank} & 2.13 & 6.68 & \textbf{\textcolor{blue}{50.80}}  & 2.22 & 7.68 & \textbf{\textcolor{black}{40.53}}  \\
\hline
\textsc{Seaquest} & 1680.84 & \textbf{\textcolor{blue}{2274.06}} & 1180.70  & 1697.19 & 2015.55 & \textbf{\textcolor{black}{3035.32}}  \\
\hline
\textsc{Skiing} & -23669.98 & \textbf{\textcolor{blue}{-20066.65}} & -26402.39  & \textbf{\textcolor{black}{-20958.97}} & -22177.50 & -27972.63  \\
\hline
\textsc{Solaris} & 2156.96 & \textbf{\textcolor{blue}{2175.70}} & 805.66  & 2102.13 & \textbf{\textcolor{black}{2270.15}} & 1752.72  \\
\hline
\textsc{Space Invaders} & \textbf{\textcolor{blue}{1653.59}} & 1466.01 & 1428.94  & \textbf{\textcolor{black}{1741.27}} & 1531.64 & 1101.43  \\
\hline
\textsc{Star Gunner} & \textbf{\textcolor{blue}{55221.64}} & 52466.84 & 47557.16  & \textbf{\textcolor{black}{59218.08}} & 55233.43 & 40171.44  \\
\hline
\textsc{Surround} & -7.79 & \textbf{\textcolor{blue}{-6.99}} & -8.77  & \textbf{\textcolor{black}{-7.10}} & -7.21 & -8.19  \\
\hline
\textsc{Tennis} & \textbf{\textcolor{blue}{-12.44}} & -20.49 & -12.98  & -16.18 & -23.06 & \textbf{\textcolor{black}{-8.00}}  \\
\hline
\textsc{Time Pilot} & \textbf{\textcolor{blue}{7417.08}} & 3816.38 & 2808.92  & \textbf{\textcolor{black}{9000.91}} & 4103.00 & 4067.51  \\
\hline
\textsc{Tutankham} & \textbf{\textcolor{blue}{250.03}} & 132.67 & 70.84  & \textbf{\textcolor{black}{273.66}} & 112.14 & 75.21  \\
\hline
\textsc{Up and Down} & \textbf{\textcolor{blue}{34362.80}} & 8705.64 & 4139.20  & \textbf{\textcolor{black}{44883.40}} & 23106.24 & 5208.67  \\
\hline
\textsc{Venture} & 0.00 & 0.00 & \textbf{\textcolor{blue}{54.86}}  & \textbf{\textcolor{black}{0.00}} & \textbf{\textcolor{black}{0.00}} & \textbf{\textcolor{black}{0.00}}  \\
\hline
\textsc{Video Pinball} & 53488.73 & 35515.91 & \textbf{\textcolor{blue}{55326.08}}  & 68287.63 & \textbf{\textcolor{black}{97372.80}} & 52995.08  \\
\hline
\textsc{Wizard Of Wor} & \textbf{\textcolor{blue}{4402.10}} & 3657.65 & 1231.23  & \textbf{\textcolor{black}{4347.76}} & 3355.09 & 378.70  \\
\hline
\textsc{Yar's Revenge} & \textbf{\textcolor{blue}{19039.24}} & 12317.49 & 14236.94  & \textbf{\textcolor{black}{20006.02}} & 13398.73 & 15042.75  \\
\hline
\textsc{Zaxxon} & 121.35 & \textbf{\textcolor{blue}{7956.05}} & 2333.52  & 152.11 & \textbf{\textcolor{black}{7451.25}} & 2481.40  \\
\hline
\hline
Times Best  & 26 & 24 & 8  & 26 & 25 & 9  \\
\hline
\end{tabular}
}
\caption{Average score after 200 million training frames for A3C and  A3C+ (with $\hN_n^{-1/2}$ bonus), with a DQN baseline for comparison.\label{table:all_the_a3c_results}}
\end{table}

}

\end{document}